\documentclass[twoside]{article}

\usepackage[accepted]{aistats2021}
\usepackage{amsfonts}
\usepackage{amsmath,graphicx}
\usepackage{amsthm}
\usepackage{algorithm,algorithmic}

\usepackage{xcolor}
\usepackage{caption}
\usepackage{subcaption}
\usepackage{multirow,array}
\usepackage{rotating}
\usepackage{mathtools}

\DeclareMathOperator*{\argmax}{arg\,max}

\newcommand{\Ex}{\mathbb{E}\hspace{0.05cm}}

\newcommand{\bs}{\boldsymbol}
\newcommand{\ba}{\left[\begin{array}}
	\newcommand{\ea}{\\\end{array} \right]}

\newcommand{\define}{\stackrel{\Delta}{=}}

\newtheorem{definition}{Definition}

\newtheorem{theorem}{Theorem}
\newtheorem{lemma}{Lemma}
\newtheorem{proposition}{Proposition}

\DeclarePairedDelimiter{\ceil}{\lceil}{\rceil}
\DeclarePairedDelimiter{\floor}{\lfloor}{\rfloor}
\begin{document}

%

%

\twocolumn[

\aistatstitle{Logical Team Q-learning: An approach towards factored policies in cooperative MARL}

\aistatsauthor{ Lucas Cassano \And  Ali H. Sayed }

\aistatsaddress{ \'Ecole Polytechnique F\'ed\'erale de Lausanne \And \'Ecole Polytechnique F\'ed\'erale de Lausanne } ]

\begin{abstract}
  We address the challenge of learning factored policies in cooperative MARL scenarios. In particular, we consider the situation in which a team of agents collaborates to optimize a common cost. The goal is to obtain factored policies that determine the individual behavior of each agent so that the resulting joint policy is optimal. The main contribution of this work is the introduction of \textit{Logical Team Q-learning (LTQL)}. LTQL does not rely on assumptions about the environment and hence is generally applicable to any collaborative MARL scenario. We derive LTQL as a stochastic approximation to a dynamic programming method we introduce in this work. We conclude the paper by providing experiments (both in the tabular and \textit{deep} settings) that illustrate the claims.
\end{abstract}

\section{INTRODUCTION}
\label{sec:intro}
Reinforcement Learning (RL) has had considerable success in many domains. In particular, Q-learning \cite{watkins1992q} and its deep learning extension DQN \cite{dqn} have shown great performance in challenging domains such as the Atari Learning Environment \cite{bellemare2013arcade}. At the core of DQN lie two important features: the ability to use expressive function approximators (in particular, neural networks (NN)) that allow it to estimate complex $Q$-functions; and the ability to learn off-policy and use replay buffers \cite{lin1992self}, which allows DQN to be very sample efficient. Traditional RL focuses on the interaction between one agent and an environment. However, in many cases of interest, a multiplicity of agents will need to interact with a unique environment and with each other. This is the object of study of Multi-agent RL (MARL), which goes back to the early work of \cite{tan1993multi} and has seen renewed interest of late (for an updated survey see \cite{zhang2019multi}). In this paper we consider the particular case of cooperative MARL in which the agents form a \textit{team} and have a shared unique goal. We are interested in tasks where collaboration is fundamental and a high degree of coordination is necessary to achieve good performance. In particular, we consider two scenarios.

In the first scenario, the global state and all actions are visible to all agents. One example of this situation could be a team of robots that collaborate to move a big and heavy object. It is well known that in this scenario the team can be regarded as one single agent where the aggregate action consists of the joint actions by all agents \cite{littman2001value}. The fundamental drawback of this approach is that the joint action space grows exponentially in the number of agents and the problem quickly becomes intractable \cite{kok2004sparse,guestrin2002coordinated}. One well-known and popular approach to solve these issues, is to consider each agent as an independent learner (IL) \cite{tan1993multi}. However, this approach has a number of problems. First, from the point of view of each IL, the environment is non-stationary (due to the changing policies of the other agents), which jeopardizes convergence. And second, replay buffers cannot be used due to the changing nature of the environment and therefore even in cases where this approach might work, the data efficiency of the algorithm is negatively affected. Ideally, it is desirable to derive an algorithm with the following features: i) it learns individual policies (and is therefore scalable), ii) local actions chosen greedily with respect to these individual policies result in an optimal team action, iii) can be combined with NN's, iv) works off-policy and can leverage replay buffers (for data efficiency), v) and enjoys theoretical guarantees to team optimal policies at least in the dynamic programming scenario. Indeed, the main contribution of this work is the introduction of \textit{Logical Team Q-learning (LTQL)}, an algorithm that has all these properties. We start in the dynamic programing setting and derive equations that characterize the desired solution. We use these equations to define the \textit{Factored Team Optimality Bellman Operator} and provide a theorem that characterizes the convergence properties of this operator. A stochastic approximation of the dynamic programming setting is used to obtain the tabular and deep versions of our algorithm. For the single agent setting, these steps reduce to: the Bellman optimality equation, the Bellman optimality operator and Q-learning (in its tabular form and DQN).


In the second scenario, we consider the centralized training and decentralized execution paradigm under partial observability. Under this scheme, training is done in a centralized manner and therefore we assume global information to be available during training. During execution, agents only have access to their own  observations. Therefore, even though during training global information is available, the learned policies must only rely on local observations. An example of this case would be a soccer team that during training can rely on a centralized server where data is aggregated but has to play games in a fully decentralized manner without the aid of such server.

\subsection{Relation to Prior Work}
\label{sec:prior_work}
Some of the earliest works on MARL are \cite{tan1993multi,claus1998dynamics}. \cite{tan1993multi} studied \textit{Independent Q-learning (IQL)} and identified that IQL learners in a MARL setting may fail to converge due to the non-stationarity of the perceived environment. \cite{claus1998dynamics} compared the performance of IQL and \textit{joint action learners (JAL)} where all agents learn the $Q$-values for all the joint actions, and identified the problem of coordination during decentralized execution when multiple optimal policies are available. \cite{littman2001value} later provided a proof of convergence for JALs. Recently, \cite{tampuu2017multiagent} did an experimental study of ILs using DQNS in the Atari game Pong. All these mentioned approaches cannot use experience replay due to the non-stationarity of the preceived environment. Following \textit{Hyper Q-learning} \cite{tesauro2004extending}, \cite{foerster2017stabilising} addressed this issue to some extent using \textit{fingerprints} as proxys to model other agents' strategies.

\cite{lauer2000riedmiller} introduced \textit{Distributed Q-learning (DistQ)}, which in the tabular setting has guaranteed convergence to an optimal policy for deterministic MDPs. However, this algorithm performs very poorly in stochastic scenarios and becomes divergent when combined with function approximation. Later \textit{Hysteretic Q-learning (HystQ)} was introduced in \cite{matignon2007hysteretic} to improve these two limitations. HystQ is based on a heuristic and can be thought of as a generalization of DistQ. These works also consider the scenario where agents cannot perceive the actions of other agents. They are related to LTQL (from this work) in that they can be considered approximations to our algorithm in the scenario where agents do not have information about other agents' actions. Recently \cite{omidshafiei2017deep} introduced \textit{Dec-HDRQNs} for multi-task MARL, which combines HystQ with Recurrent NNs and experience replay (which they recognize is important to achieve high sample efficiency) through the use of \textit{Concurrent Experience Replay Trajectories}.

\cite{wang2003reinforcement} introduced \textit{OAB}, the first algorithm that converges to an optimal Nash equilibrium with probability one in any team Markov game. OAB considers the team scenario where agents observe the full state and joint actions. The main disadvantage of this algorithm is that it requires estimation of the transition kernel and rewards for the joint action state space and also relies on keeping count of state-action visitation, which makes it impractical for MDPs of even moderate size and cannot be combined with function approximators. 

\cite{guestrin2002multiagent,guestrin2002coordinated,kok2004sparse} introduced the idea of factoring the joint $Q$-function to handle the scalability issue. These papers have the disadvantage that they require coordination graphs that specify how agents affect each other (the graphs require significant domain knowledge). The main shortcoming of these papers is the factoring model they use, in particular they model the optimal $Q$-function (which depends on the joint actions) as a sum of $K$ local $Q$-functions (where $K$ is the number of agents, and each $Q$-function considers only the action of its corresponding agent). The main issue with this factorization model is that the optimal $Q$-function cannot always be factored in this way, in fact, the tasks for which this model does not hold are typically the ones that require a high degree of coordination, which happen to be the tasks where one is most interested in applying specific MARL approaches as opposed to ILs. The approach we introduce in this paper also considers learning factored $Q$-functions. However, the fundamental difference is that the factored relations we estimate always exist and the joint action that results from maximizing these individual $Q$-functions is optimal. \textit{VDN} \cite{sunehag2018value} and \textit{Qmix} \cite{rashid2018qmix} are two recent \textit{deep} methods that also factorize the optimal $Q$-function assuming additivity and monotonicity, respectively. This factoring is their main limitation since many MARL problems of interest do not satisfy any of these two assumptions. Indeed, \cite{son2019qtran} showed that these methods are unable to solve a simple matrix game. Furthermore, the individual policies cannot be used for prediction, since the individual $Q$ values are not estimates of the return. To improve on the representation limitation due to the factoring assumption, \cite{son2019qtran} introduced \textit{QTRAN} which factors the $Q$-function in a more general manner and therefore allows for a wider applicability. The main issue with QTRAN is that although it can approximate a wider class of $Q$-functions than VDN and Qmix, the algorithm resorts to other approximations, which degrade its performance in complex environments (see \cite{rashid2020monotonic}).

Recently, actor-critic strategies have been explored. The algorithm introduced in \cite{zhang2018fully} has the disadvantage that it performs poor credit assignment and as a consequence can easily converge to highly suboptimal strategies (see \cite{cassano2019team}). \cite{gupta2017cooperative} introduces policy gradient schemes that also have the credit assignment issue. The algorithm presented by \cite{foerster2018counterfactual} addresses this issue, but does so by learning the team's \textit{joint} $q$-function and hence this approach does not address the exponential scalability issue. These methods have the added inconvenience that they are on-policy and hence do not enjoy the data efficiency that off-policy methods can achieve.

\section{PROBLEM FORMULATION}
\label{sec:problem}
We consider a situation where multiple agents form a team and interact with an environment and with each other. We model this interaction as a \textit{decentralized partially observable Markov decision process} (Dec-POMDP)\cite{oliehoek2016concise}, which is defined by the tuple ($\mathcal{S}$,$\mathcal{K}$,$o^k$,$\mathcal{A}^k$,$\mathcal{P}$,$r$), where, $\mathcal{S}$ is a set of global states shared by all agents; $\mathcal{K}$ is the set of $K$ of agents; $o^k:\mathcal{S}\rightarrow\mathcal{O}^k$ is the observation function for agent $k$, whose output lies in some set of observations $\mathcal{O}^k$; $\mathcal{A}^k$ is the set of actions available to agent $k$; $\mathcal{P}(s'|s,a^1,\cdots,a^K)$ specifies the probability of transitioning to state $s'\in\mathcal{S}$ from state $s\in\mathcal{S}$ having taken joint actions $a^k\in\mathcal{A}^{k}$; and $r:\mathcal{S}\times\mathcal{A}^{1}\times\cdots\times\mathcal{A}^{K}\times\mathcal{S}\rightarrow\mathbb{R}$ is a global reward function. Specifically, $r(s,a^1,\cdots,a^K,s')$ is the reward when the team transitions to state $s'\in\mathcal{S}$ from state $s\in\mathcal{S}$ having taken actions $a^1,\cdots,a^K$. The reward $r(s,a^1,\cdots,a^K,s')$ can be a random variable following some distribution $f_{s,a^1,\cdots,a^K,s'}(r)$. We clarify that from now on we will refer to the collection of all individual actions as the team's action, denoted as $\bar{a}$. Furthermore we will use $a^{-k}$ to refer to the actions of all agents except for action $a^k$. Therefore we can write $\mathcal{P}(s'|s,a^1,\cdots,a^K)=\mathcal{P}(s'|s,a^k,a^{-k})=\mathcal{P}(s'|s,\bar{a})$. The goal of the team is to maximize the team's return:
\begin{align}\label{eq:opt_problem}
J(\pi)&=\sum_{t=0}^{\infty}\gamma^t\Ex_{\pi,\mathcal{P},d,f}\left[\bs{r}(\bs{s_{t}},\bs{\bar{a}_t},\bs{s_{t+1}})\right]
\end{align}
where $s_{t}$ and $\bar{a}_t$ are the state and actions at time $t$, respectively, $\pi(\bar{a}|s)$ is the team's policy, $d$ is the distribution of initial states, and $\gamma\in[0,1)$ is the discount factor. We clarify that we use bold font to denote random variables and the notation $\Ex_{\ell}$ makes explicit that the expectation is taken with respect to distribution $\ell$. From now on, we will only make the distributions explicit in cases where doing so makes the equations more clear. Accordingly, the team's optimal state-action value function ($q^\dagger$) and optimal policy ($\pi^\dagger$) are given by \cite{sutton1998reinforcement}:
\begin{subequations}\label{eq:opt_single_policy}
	\begin{align}
	\pi^\dagger\hspace{-0.3mm}(\bar{a}|s)&\hspace{-0.6mm}=\hspace{-0.6mm}\argmax_{\pi(\bar{a}|s)}\Ex_{\pi,\mathcal{P}}\hspace{-0.5mm}\big[r(s,\bs{\bar{a}})\hspace{-0.5mm}+\hspace{-0.5mm}\gamma\max_{\bar{a}^{\prime}}q^\dagger(\bs{s'}\hspace{-0.4mm},\bar{a}^{\prime})\big]\label{eq:opt_policy}\\
	q^\dagger\hspace{-0.3mm}(s,\bar{a})&\hspace{-0.5mm}=\hspace{-0.5mm}\Ex_{\mathcal{P}}\big[r(s,\bar{a})+\gamma\max_{\bar{a}^{\prime}}q^\dagger(\bs{s'},\bar{a}^{\prime})\big]\label{eq:bellman_opt}
	\end{align}
\end{subequations}
where $r(s,\bar{a})=\Ex_{\mathcal{P},f}\bs{r}(s,\bar{a},\bs{s'})$. As already mentioned, a team problem of this form can be addressed with any single-agent algorithm. The fundamental inconvenience with this approach is that the joint action space scales exponentially with the number of agents, more specifically $|\bar{\mathcal{A}}|=\prod_{k=1}^K|\mathcal{A}^{k}|$ (where $\bar{\mathcal{A}}$ is the joint action space). Another problem with this approach is that the learned $Q$-function cannot be executed in a decentralized manner using the agents' observations. For these reasons, in the next sections we concern ourselves with learning factored quantities.

\section{DYNAMIC PROGRAMMING}\label{section:dynamic_programming}
Similarly to the way that relation \eqref{eq:bellman_opt} is used to derive $Q$-learning in the single agent setting, the goal of this section is to derive relations in the dynamic programming setting from which we can derive a cooperative MARL algorithm. The following two propositions take the first steps in this direction.
\begin{proposition}\label{proposition:factored_optimal_q}
	For each deterministic team optimal policy, there exist $K$ factored functions $q^{k,\star}:\mathcal{S}\times\mathcal{A}^k\rightarrow\mathbb{R}$ such that:
	\begin{subequations}\label{eq:conditions_t_factored_q}
		\allowdisplaybreaks
		\begin{align}
		&\max_{\bar{a}}q^\dagger(s,\bar{a})\hspace{-0.5mm}=\hspace{-0.5mm}\max_{a^k}q^{k,\star}(s,a^k),\hspace{5mm}\forall 1\leq k\leq K\label{eq:factored_q_1}\\
		&\max_{\bar{a}}q^\dagger(s,\bar{a})\hspace{-0.5mm}=\hspace{-0.5mm}q^\dagger\Bigl(s,\argmax_{a^1}q^{1,\star}(s,a^1),\cdots\nonumber\\
		&\hspace{35mm},\argmax_{a^K}q^{K,\star}(s,a^K)\Bigr)\label{eq:factored_q_2}\\
		&q^{k,\star}(s,a^k)\hspace{-0.5mm}=\hspace{-0.5mm}\mathcal{B}_{E}q^{k,\star}(s,a^k)\label{eq:optimal_fixed_point_1}\\
		&\mathcal{B}_{E}q^{k}(s,a^k)=r(s,a^k,a^{-k})\nonumber\\
		&\hspace{5mm}+\hspace{-0.5mm}\gamma\Ex\max_{a^{\prime}}q^{k}(\bs{s'},a^{\prime})\big|_{a^{n}=\argmax\limits_{a^n}q^{n}(s,a^n)\hspace{1mm}\forall n\neq k}\label{eq:Be}
		\end{align}
	\end{subequations}
where operator $\mathcal{B}_{E}$ is defined such that if there are multiple arguments that maximize $\argmax_{a^n}q^{n}(s,a^n)$ the actions that jointly maximize \eqref{eq:Be} are chosen.
\end{proposition}
\begin{proof}
	Assume that we have some deterministic team optimal policy $\pi^\dagger(\bar{a}|s)$. We define $q^{k,\star}(s,a^k)$ as follows:
	\begin{align}\label{eq:q_star}
	q^{k,\star}(s,a^k)&=q^{\dagger}(s,a^k,a^{-k})|_{\argmax_{a^{-k}}\pi^\dagger(a^k,a^{-k}|s)}
	\end{align}
	Note that by construction, $q^{k,\star}(s,a^k)$ satisfies relations \eqref{eq:factored_q_1} and \eqref{eq:factored_q_2} and also:
	\begin{align}
	&\argmax_{a^k}q^{k,\star}(s,a^k)=a^k\sim\pi^\dagger(a^k,a^{-k}|s)\hspace{2mm}\forall k\label{eq:q_star_3}
	\end{align}
	Relation \eqref{eq:optimal_fixed_point_1} is obtained by combining relations \eqref{eq:bellman_opt}, \eqref{eq:q_star} and \eqref{eq:q_star_3}.
\end{proof}
A simple interpretation of equation \eqref{eq:optimal_fixed_point_1} is that $q^{k,\star}(s,a^k)$ is the expected return starting from state $s$ when agent $k$ takes action $a^k$ while the rest of the team acts in an optimal manner. 
\begin{proposition}\label{proposition:factored_optimal_policy}
	Each deterministic team optimal policy that can be factored into $K$ deterministic policies $\pi^{k,\star}(a^k|s)$. Such factored deterministic policies can be obtained as follows:
	\begin{align}\label{eq:factored_optimal_policy}
	\pi^{k,\star}(a^k|s)&\textstyle=\mathbb{I}\big(a^k\hspace{-0.7mm}=\hspace{-0.7mm}\argmax_{a^k}q^{k,\star}(s,a^k)\big)
	\end{align}
	where $\mathbb{I}$ is the indicator function.
\end{proposition}
\begin{proof}
	The proof follows from equations \eqref{eq:factored_q_1}-\eqref{eq:factored_q_2}.
\end{proof}
Propositions \ref{proposition:factored_optimal_q} and \ref{proposition:factored_optimal_policy} are useful because they show that if the agents learn factored functions that satisfy \eqref{eq:conditions_t_factored_q} and act greedily with respect to their corresponding $q^{k,\star}$, then the resulting team policy is guaranteed to be optimal and hence they are not subject to the coordination problem identified in \cite{lauer2000riedmiller}\footnote{This problem arises in situations in which the environment has multiple deterministic team optimal policies and the agents learn factored functions of the form $\max_{a^{-k}q^\dagger(s,a^k,a^{-k})}$ (we remark that these are not the same as $q^{k,\star}(s,a^k)$).} (we show this in section \ref{subsec:matrix_game}). Therefore, an algorithm that learns $q^{k,\star}$ would satisfy the first two of the five desired properties that were enumerated in the introduction. As a sanity check, note that for the case where there is only one agent, equation \eqref{eq:optimal_fixed_point_1} simplifies to the Bellman optimality equation. Although in the single agent case the Bellman optimality operator can be used to obtain $q^\dagger$ (by repeated application of the operator), we cannot do the same with $\mathcal{B}_E$. The fundamental reason for this is stated in proposition \ref{proposition:nash}.
\begin{proposition}\label{proposition:nash}
	\textbf{Sub-optimal Nash fixed points:} There may exist $K$ functions $q^{k}$ such that \eqref{eq:optimal_fixed_point_1} is satisfied but \eqref{eq:factored_q_2} is not.
\end{proposition}
\begin{proof}
	See Appendix 6.1.
\end{proof}
Note that proposition \ref{proposition:nash} implies that relation \eqref{eq:optimal_fixed_point_1} is not sufficient to derive a learning algorithm capable of obtaining a team optimal policy because it can converge to sub-optimal team strategies instead. To avoid this inconvenience, it is necessary to find another relation that is only satisfied by $q^\star$. We can obtain one such relation combining \eqref{eq:factored_q_1} and \eqref{eq:optimal_fixed_point_1}:
\begin{align}
\max_{a^k}q^{k\hspace{-0.1mm},\hspace{-0.1mm}\star}\hspace{-0.1mm}(\hspace{-0.2mm}s,\hspace{-0.5mm}a^k)&\hspace{-0.6mm}=\hspace{-0.6mm}\max_{\bar{a}}\hspace{-0.5mm}\big[r(\hspace{-0.2mm}s,\hspace{-0.3mm}\bar{a})\hspace{-0.7mm}+\hspace{-0.7mm}\gamma\Ex\hspace{-0.8mm}\max_{a^{\prime,k}}q^{k\hspace{-0.1mm},\hspace{-0.1mm}\star}(\hspace{-0.2mm}\bs{s'}\hspace{-1.2mm},\hspace{-0.2mm}a^{\prime,k})\big]\label{eq:optimal_fixed_point_2}
\end{align}
The sub-optimal Nash fixed points mentioned in proposition \ref{proposition:nash} do not satisfy relation \eqref{eq:optimal_fixed_point_2} since by definition the right hand side is equal to $\max_{\bar{a}}q^\dagger(s,\bar{a})$. Intuitively, equation \eqref{eq:optimal_fixed_point_2} is not satisfied by these suboptimal strategies because the $\max_{\bar{a}}$ considers all possible team actions (while Nash equilibria only consider unilateral deviations).
\begin{definition}\label{definition:bellman_optimality}
	Combining equations \eqref{eq:optimal_fixed_point_1} and \eqref{eq:optimal_fixed_point_2}, we define the Factored Team Optimality Bellman operator $\mathcal{B}_{p}$ as follows:
	\begin{align}
	&\mathcal{B}_p\hspace{-0.1mm}q^{k}\hspace{-0.2mm}(\hspace{-0.2mm}s,\hspace{-0.5mm}a^k)\hspace{-0.7mm}=\hspace{-0.7mm}\begin{cases}
	\hspace{-0.5mm}\mathcal{B}_{\hspace{-0.5mm}E}q^{k}(\hspace{-0.2mm}s,\hspace{-0.5mm}a^k)\hspace{1mm}\text{with probability }p\hspace{-0.7mm}>\hspace{-0.7mm}0\\
	\hspace{-0.5mm}\mathcal{B}_{\hspace{-0.5mm}I}q^{k}(\hspace{-0.2mm}s,\hspace{-0.5mm}a^k)\hspace{2mm}\text{else}
	\end{cases}\label{eq:bu}\\
	&\mathcal{B}_{I}q^{k}(s,a^k)=\max\big\{q^{k}(s,a^k),\nonumber\\
	&\hspace{10mm}\max_{a^{-k}}\big(r(s,a^k,a^{-k})+\gamma\Ex\max_{a^{\prime}}q^{k}(\bs{s'},a^{\prime})\big)\big\}
	\end{align}
\end{definition}
LTQL is based on operator $\mathcal{B}_p$, the reason we use relations \eqref{eq:optimal_fixed_point_1} and \eqref{eq:optimal_fixed_point_2} to define this operator instead of just \eqref{eq:optimal_fixed_point_2}, is that using only relation \eqref{eq:optimal_fixed_point_2} we would derive DistQ that has the shortcomings we discussed in section \ref{sec:prior_work}. A simple interpretation of operator $\mathcal{B}_p$ is the following. Consider a basketball game, in which player $\alpha$ has the ball and passes the ball to teammate $\beta$. If $\beta$ gets distracted, misses the ball and the opposing team ends up scoring, should $\alpha$ learn from this experience and modify its policy to not pass the ball? The answer is no, since the poor outcome was player $\beta$'s fault. In plain English, from the point of view of some player $k$, what the operator $\mathcal{B}_E$ means is \textit{``I will only learn from experiences in which my teammates acted according to what I think is the optimal team strategy"}. It is easy to see why this kind of stubborn rationale cannot escape Nash equilibria (i.e., agents do not learn when the team deviates from its current best strategy, which obviously is a necessary condition to learn better strategies). The interpretation of the full operator $\mathcal{B}_p$ is \textit{``I will learn from experiences in which: a) my teammates acted according to what I think is the optimal team strategy; or b) my teammates deviated from what I believe is the optimal strategy and the outcome of such deviation was better than I expected if they had acted according to what I thought was optimal"}, which arguably is what a logical player would do (this is the origin of the algorithm's name). We now proceed to describe the convergence properties of operator $\mathcal{B}_p$.
\begin{lemma}\label{lemma:over_estimate}
	For any $\delta_1>0$, after $N$ applications of operator $\mathcal{B}_p$ to any set of $K$ $q^{k}(s,a^k)$ functions it holds:
	\begin{align}
	&\mathbb{P}\big(\mathcal{B}_p^Nq^{k}(s,a^k)\in\mathcal{C}_{\delta_1}^U\big)\geq1-\sum_{n=0}^{n_o}{{N}\choose{n}}p^n(1-p)^{N-n}\label{eq:cdf_binomial}\\
	&\mathcal{C}_{\delta_1}^U=\big\{q^{k}|q^{k}(s,a^k)\leq\max_{a^{-k}}q^\dagger(s,a^k,a^{-k})+\delta_1\nonumber\\
	&\hspace{35mm}\forall (k,s,a^k)\hspace{-1mm}\in\hspace{-1mm}(\mathcal{K},\mathcal{S},\mathcal{A}^{k})\big\}\\
	&n_o=\floor[\bigg]{\log_\gamma\left(\frac{\delta_1}{q_{U}-\min_{s}\max_{\bar{a}}q^\dagger(s,\bar{a})}\right)}\\
	&q_U=\max\{r_{\max}(1-\gamma)^{-1},\max_{k,s,a^k}q^k(s,a^k)\}
	\end{align}
	where $r_{\max}=\max_{s,\bar{a}}r(s,\bar{a})$. For the special case where $N>n_o/p$ we can lower bound equation \eqref{eq:cdf_binomial} as follows:
	\begin{align}
	&\mathbb{P}\big(\mathcal{B}_p^Nq^{k}(s,a^k)\in\mathcal{C}_{\delta_1}^U\big)\geq 1-e^{-2N\left(p-\frac{n_o}{N}\right)^2}
	\end{align}
\end{lemma}
\begin{proof}
	See Appendix 6.2.
\end{proof}
\begin{lemma}\label{lemma:approach}
	After $N\geq L$ applications of operator $\mathcal{B}_p$ to any set of $K$ $q^{k}(s,a^k)\in\mathcal{C}_0^U$ functions it holds:
	\begin{align}
	&\mathbb{P}\big(\mathcal{B}_p^Nq^{k}(s,a^k)\in\mathcal{C}_{\delta_2}\big)\hspace{-0.7mm}\geq\hspace{-0.7mm}1\hspace{-0.7mm}-\hspace{-0.7mm}\beta_{N,L}\hspace{-0.7mm}+\hspace{-0.7mm}(1\hspace{-0.7mm}-\hspace{-0.7mm}p)^L\beta_{N\hspace{-0.3mm}-\hspace{-0.3mm}L,L}\label{eq:coin_toss}\\
	&\beta_{N,L}=\sum_{j=0}^{\floor{N/(L+1)}}(-1)^j{{N-jL}\choose{j}}\big(p(1-p)^L\big)^j\\
	&L\hspace{-0.6mm}=\hspace{-0.6mm}\ceil[\Bigg]{\hspace{-0.6mm}\log_\gamma\hspace{-0.9mm}\Bigg(\hspace{-0.5mm}\frac{\delta_2}{\max\limits_{s}\big|\max\limits_{a^k}q^{k}(s,a^k)\hspace{-0.3mm}-\hspace{-0.3mm}\max\limits_{\bar{a}}q^\dagger(s,\bar{a})\big|}\hspace{-0.5mm}\Bigg)\hspace{-0.6mm}}\\
	&\mathcal{C}_{\delta_2}=\big\{q^{k}|q^{k,\star}(s,a^k)-\delta_2\leq q^{k}(s,a^k)\nonumber\\
	&\leq\hspace{-0.8mm}\max_{a^{-k}}q^\dagger(s,a^k,a^{-k})\hspace{-0.5mm}+\hspace{-0.5mm}\delta_2\forall (k,s,a^k)\hspace{-1mm}\in\hspace{-1mm}(\mathcal{K},\mathcal{S},\mathcal{A}^{k})\big\}
	\end{align}
	for any $\delta_2>0$. If $p>0.5$, probability \eqref{eq:coin_toss} can be bounded by:
	\begin{align}
	&\mathbb{P}\big(\mathcal{B}_p^Nq^{k}(s,a^k)\in\mathcal{C}_{\delta_2}\big)\geq 1-\frac{1-(1-p)\xi_1}{p\xi_1(1+L-L\xi_1)}\xi_1^{-N}\nonumber\\
	&\hspace{38.5mm}-\frac{L}{p}(1-p)^{N+2}
	\end{align}
	where $1<\xi_1<1+L^{-1}$.
\end{lemma}
\begin{proof}
	See Appendix 6.3.
\end{proof}
Lemma \ref{lemma:over_estimate} indicates that if the initial functions $q^k(s,a^k)$ have values that are larger than $\max_{a^{-k}}q^\dagger(s,a^k,a^{-k})$, after sufficient applications of operator $\mathcal{B}_p$ all overestimations will be reduced such that $q^{k}(s,a^k)\leq\max_{a^{-k}}q^\dagger(s,a^k,a^{-k})+\delta_1$ with high probability. Lemma \ref{lemma:approach} show that if the operator $\mathcal{B}_p$ is applied sufficient times to functions that do not overestimate $\max_{a^{-k}}q^\dagger(s,a^k,a^{-k})$, then the obtained function lie in a small neighborhood of the desired solution with high probability. These results give rise to the following important theorem.
\begin{theorem}\label{theorem:convergence}
	Repeated application of the operator $\mathcal{B}_p$ to any initial set of $K$ $q^k$-functions followed by an application of operator $\mathcal{B}_E$ converge to the $\delta$-neighborhood ($\delta>0$) of some set $q^{k,\star}$ with high probability. For the particular case, where $p>0/5$ and $N>L+n_o/p$ it holds:
	\begin{align}\label{eq:theorem}
	\mathbb{P}\big(|\mathcal{B}_E\mathcal{B}_p^Nq^{k}(s,a^k)\hspace{-0.7mm}-\hspace{-0.7mm}q^{k,\star}(s,a^k)|\hspace{-1mm}<\hspace{-1mm}\delta\big)\hspace{-1mm}\geq\hspace{-1mm}1\hspace{-0.7mm}-\hspace{-0.7mm}\mathcal{O}(\theta^N)
	\end{align}
	for any $\delta>0$, where $0\leq\theta<1$ is a constant that depends on $\delta$, $\gamma$, $p$, $r(s,\bar{a})$, $\mathcal{P}$ and the initial functions $q^k(s,a^k)$.
\end{theorem}
\begin{proof}
	Combining the results from lemmas \ref{lemma:over_estimate} and \ref{lemma:approach} and setting $\delta=\delta_1=\delta_2$ we get that after $N_1+N_2>L+n_o/p>0$ applications of operator $\mathcal{B}_p$ to any set of $K$ $q^{k}(s,a^k)$ functions it holds:
	\begin{align}
	&\mathbb{P}\big(\mathcal{B}_p^Nq^{k}(s,a^k)\in\mathcal{C}_{\delta}\big)\geq\max_{\substack{N_1>\frac{n_o}{p}\\N_2\geq L}}\left(1-e^{-2N_1\left(p-\frac{n_o}{N_1}\right)^2}\right)\nonumber\\
	&\cdot\left(1-\frac{1-(1-p)\xi_1}{p\xi_1(1+L-L\xi_1)}\xi_1^{-N_2}-\frac{L}{p}(1-p)^{N_2+2}\right)\nonumber\\
	&=1-\mathcal{O}(\theta^N)\label{eq:final_prob}
	\end{align}
	where $0\leq\theta<1$. Now we proceed to analyze $\mathcal{B}_E\mathcal{B}_p^Nq^{k}(s,a^k)$. If $q^{k}(s,a^k)\in\mathcal{C}_{\delta}$ and $\delta$ satisfies:
	\begin{align}
	\delta<\frac{1}{2}\min_{s}\big(\max_{\bar{a}}q^\dagger(s,\bar{a})-\max_{\bar{a}\neq\argmax_{\bar{a}}q^\dagger(s,\bar{a})}\hspace{-5mm}q^\dagger(s,\bar{a})\big)\label{eq:delta_cond}
	\end{align}
	we get:
	\begin{align}\label{eq:be_opt_main}
	&\mathcal{B}_{E}q^{k}(s,a^k)=\big(r(s,a^k,a^{-k})\nonumber\\
	&\hspace{0mm}+\gamma\Ex\max_{a^{\prime}}q^{k}(\bs{s'},a^{\prime})\big)\big|_{a^{n}=\argmax_{a^n}q^{n}(s,a^n)\hspace{1mm}\forall n\neq k}
	\end{align}
	Using the fact that $q^{k}(s,a^k)\in\mathcal{C}_{\delta}$ it follows:
	\begin{align}
	&\max_{a^k}q^{k,\star}(s,\bar{a})-\delta\leq \max_{a^k}q^{k}(s,a^k)\\
	&q^{k}(s,a^{k,\bullet})\leq\max_{a^{-k}}q^\dagger(s,a^{k,\bullet},a^{-k})+\delta\nonumber\\
	&\hspace{7mm}\stackrel{(a)}{<}\max_{\bar{a}}q^\dagger(s,\bar{a})-\delta\stackrel{(b)}{=}\max_{a^k}q^{k,\star}(s,\bar{a})-\delta_2\\
	&a^{k,\bullet}=\argmax_{a^k\neq\argmax_{a^k}q^{k}(s,a^k)}q^{k}(s,a^k)\label{eq:correct_act_main}
	\end{align}
	where in $(a)$ we used condition \eqref{eq:delta_cond} and in $(b)$ we used equation \eqref{eq:factored_q_1}. Combining equations \eqref{eq:be_opt_main} through \eqref{eq:correct_act_main} we get:
	\begin{align}
	&\mathcal{B}_{E}q^{k}(s,a^k)=\big(r(s,a^k,a^{-k})\nonumber\\
	&\hspace{3mm}+\gamma\Ex\max_{a^{\prime}}q^{k}(\bs{s'},a^{\prime})\big)\big|_{a^{n}=\argmax\limits_{a^n}q^{n,\star}(s,a^n)\hspace{1mm}\forall n\neq k}\nonumber\\
	&=\big(r(s,a^k,a^{-k})+\gamma\Ex\big(\max_{a^{\prime}}q^{k,\star}(\bs{s'}\hspace{-0.5mm},a^{\prime})+\max_{a^{\prime}}q^{k}(\bs{s'}\hspace{-0.5mm},a^{\prime})\nonumber\\
	&\hspace{3mm}-\max_{a^{\prime}}q^{k,\star}(\bs{s'},a^{\prime})\big)\big)\big|_{a^{n}=\argmax\limits_{a^n}q^{n,\star}(s,a^n)\hspace{1mm}\forall n\neq k}\nonumber\\
	&=q^{k,\star}(s,a^k)+\gamma\Ex\big(\max_{a^{\prime}}q^{k}(\bs{s'},a^{\prime})\nonumber\\
	&\hspace{3mm}-\max_{a^{\prime}}q^{k,\star}(\bs{s'},a^{\prime})\big)\big|_{a^{n}=\argmax\limits_{a^n}q^{n,\star}(s,a^n)\hspace{1mm}\forall n\neq k}\label{eq:be_bon_main}
	\end{align}
	Combining equation \eqref{eq:be_bon_main} with the fact that $q^{k}(s,a^k)\in\mathcal{C}_{\delta}$ we get:
	\begin{align}\label{eq:end}
	&|\mathcal{B}_{E}q^{k}(s,a^k)-q^{k,\star}(s,a^k)|\leq\gamma\delta
	\end{align}
	Combining \eqref{eq:end} and \eqref{eq:final_prob} completes the proof.
\end{proof}
Relation \eqref{eq:end} shows why we include an application of $\mathcal{B}_E$ at the end in equation \eqref{eq:theorem}. The reason is that if we do not, the $q$-value for \textit{suboptimal actions} oscillates between $q^{k,\star}(s,a^k)$ and $\max_{a^{-k}}q^\dagger(s,a^k,a^{-k})$, we illustrate this effect in appendix 6.5. We reiterate that $q^{k,\star}(s,a^k)$ and $\max_{a^{-k}}q^\dagger(s,a^k,a^{-k})$ are only equal when optimal actions are chosen (equation \eqref{eq:factored_q_1}). $q^{k,\star}(s,a^k)$ is the expected return if agent $k$ chooses action $a^k$ and the rest of the team follows an optimal policy, while $\max_{a^{-k}}q^\dagger(s,a^k,a^{-k})$ is the best return that can be achieved if agent $k$ chooses action $a^k$. 

\section{REINFORCEMENT LEARNING}
In this section we present LTQL (see algorithm \ref{algorithm:deep_ltq}), which we obtain as a stochastic approximation to the procedure described in theorem \ref{theorem:convergence}. Note that the algorithm utilizes two $q$ estimates for each agent $k$, a biased one parameterized by $\theta^k$ (which we denote $q_{\theta^k}$) and an unbiased one parameterized by $\omega^k$ (which we denote $q_{\omega^k}$). We clarify that in the listing of LTQL $+\hspace{-1.0mm}=$ is the \textit{accumulate and add} operator and that we used a constant step-size, however this can be replaced with decaying step-sizes or other schemes such as \textit{AdaGrad} \cite{adagrad} or \textit{Adam} \cite{adam}. Note that the target of the unbiased network is used to calculate the target values for both functions; this prevents the bias in the estimates $q_{\theta^k}$ (which arises due to the $c_2$ condition)\footnote{We refer to the condition of the first \textbf{if} statement (i.e., $a^{n}=\argmax_{a^n}q_{\theta_T^{k}}(s,a^n)\hspace{1mm}\forall\hspace{-0.3mm}n\hspace{-0.5mm}\neq\hspace{-0.5mm}k$) as $c_1$, and the condition corresponding to the second \textbf{if} statement as $c_2$.} from propagating through bootstrapping. The target parameters of the biased estimates ($\theta_T^k$) are used solely to evaluate condition $c_1$. We have found that this stabilizes the training of the networks, as opposed to just using $\theta^k$. Hyperparameter $\alpha$ weights samples that satisfy condition $c_2$ differently from those who satisfy $c_1$. As we remarked in the introduction, LTQL reduces to \textit{DQN} for the case where there is a unique agent. In appendix 6.4 we include the tabular version of the algorithm along with a brief discussion.

\begin{algorithm}[t]
	\caption{\textit{Logical Team Q-Learning}}
	\label{algorithm:deep_ltq}
	\begin{algorithmic}
		\STATE{\bfseries Initialize:} an empty replay buffer $\mathcal{R}$, parameters $\theta^k$ and $\omega^k$ and their corresponding targets $\theta_T^k$ and $\omega_T^k$ for all agents $k\in\mathcal{K}$.
		\FOR{iterations $e=0,\ldots,E$}
		\STATE Sample $T$ transitions $(s,\bar{a},r,s^\prime)$ by following some behavior policy which guarantees all joint actions are sampled with non-zero probability and store them in $\mathcal{R}$.
		\FOR{iterations $i=0,\ldots,I$}
		\STATE Sample a mini-batch of $B$ transitions $(s,\bar{a},r,s^\prime)$ from $\mathcal{R}$.
		\STATE Set $\Delta_{\theta^k}=0$ and $\Delta_{\omega^k}=0$ for all agents $k$.
		\FOR{each transition of the mini-batch $b=1,\cdots,B$ and each agent $k=1,\cdots,K$}
		\IF{$a^{n}=\argmax_{a^n}q_{\theta_T^{k}}(s,a^n)\hspace{1mm}\forall\hspace{-0.3mm}n\hspace{-0.5mm}\neq\hspace{-0.5mm}k$}
		\STATE $\Delta_{\theta^{k}}+\hspace{-0.5mm}=\hspace{-0.5mm}\nabla_{\theta^{k}}\big(r+\max\limits_{a}q_{\omega_T^{k}}(s^{\prime}\hspace{-0.5mm},a)-q_{\theta^{k}}(s,a^k)\big)$
		\STATE $\Delta_{\omega^{k}}\hspace{-0.5mm}+\hspace{-0.5mm}=\hspace{-0.5mm}\nabla_{\omega^{k}}\big(r+\max\limits_{a}q_{\omega_T^{k}}(s^{\prime}\hspace{-0.5mm},a)-q_{\omega^{k}}(s,a^k)\big)$
		\ELSIF{$\big(r+\max\limits_{a}q_{\theta_T^{k}}(s^{\prime},a)>q_{\theta^{k}}(s,a^k)\big)$}
		\STATE $\Delta_{\theta^{k}}\hspace{-0.5mm}+\hspace{-0.5mm}=\hspace{-0.5mm}\alpha\hspace{-0.5mm}\nabla_{\theta^{k}}\hspace{-0.5mm}\big(r+\max\limits_{a}q_{\omega_T^{k}}(s^{\prime}\hspace{-0.5mm},a)-q_{\theta^{k}}(s,a^k)\big)$
		\ENDIF
		\ENDFOR
		\STATE $\theta^k+\hspace{-0.5mm}=\mu\Delta_{\theta^k}\hspace{10mm}\omega^k+\hspace{-0.5mm}=\mu\Delta_{\omega^k}$
		\ENDFOR
		\STATE Update targets $\theta_T^k=\theta^k$ and $\omega_T^k=\omega^k$.
		\ENDFOR
	\end{algorithmic}
\end{algorithm}

Note that LTQL works off-policy and there is no necessity of synchronization for exploration. Therefore in applications where agents have access to the global state and can perceive the actions of all other agents (so that they can evaluate $c_1$), it can be implemented in a fully decentralized manner. Interestingly, if condition $c_1$ was omitted (to eliminate the requirement that agents have access to all this information), the resulting algorithm is exactly DistQ \cite{lauer2000riedmiller}. However, as the proof of lemma \ref{lemma:over_estimate} indicates, the resulting algorithm would only converge in situations where it could be guaranteed that during learning, overestimation of the $q$ values is not possible (i.e., the tabular setting applied to deterministic environments; this remark was already made in \cite{lauer2000riedmiller}). In the case where this condition could not be guaranteed (i.e., when using function approximation and/or stochastic environments), some mechanism to decrease overestimated $q$ values would be necessary, as this is the main task of the updates due to $c_1$. One possible way to do this would be to use all transitions to update the $q$ estimates but use a smaller step-size for the ones that do not satisfy $c_2$. Notice that the resulting algorithm would be exactly HystQ \cite{matignon2007hysteretic}.

Notice that the listing of LTQL relies on global states $s$ as opposed to local agent observations. Therefore in its current form the algorithm is only applicable to the first scenario described in the introduction, in which agents have access to the global state both during training and execution. For the second scenario, in which during execution agents rely on their local observation we make the usual approximation $q^k(\mathcal{H}^k,a^k)\approx q^k(s,a^k)$ where $\mathcal{H}^k$ is the action-observation history of agent $k$. Hence, to adapt algorithm \ref{algorithm:deep_ltq} to this second scenario all that is necessary is to replace $q_{\theta^{k}}(s,a^k)$ for $q_{\theta^{k}}(\mathcal{H}^k,a^k)$ (and similarly for $\omega^{k}$, $\theta_T^{k}$ and $\omega_T^{k}$), and the observation histories need to be stored in the replay buffer as well. In practice recurrent architectures (like the \textit{Long Short Term Memory} (LSTM) \cite{hochreiter1997long}) can be used to parameterize $q_{\theta^{k}}(\mathcal{H}^k,a^k)$ as is done in \textit{Deep recurrent Q-Network (DRQN)} \cite{drqn}.
\begin{figure*}[!t]
	\vspace{.3in}
	\begin{center}
		\begin{minipage}{.32\linewidth}
			\centering
			\hspace{-5mm}
			\begin{subfigure}{\textwidth}
				\vspace{.4in}
				\hspace{5mm}
				\begin{tabular}{*{5}{c|}}
					\multicolumn{1}{c}{} & \multicolumn{1}{c}{} &\multicolumn{3}{c}{Agent $2$}\\\cline{3-5}
					\multicolumn{1}{c}{\multirow{3}*{\begin{turn}{-90}Agent 1\end{turn}}}  &  & $a_1$ & $a_2$ & $a_3$ \\\cline{2-5}
					& $b_1$ & $0$ & $2$ & $0$ \\\cline{2-5}
					& $b_2$ & $0$ & $1$ & $2$\\\cline{2-5}
				\end{tabular}
				\vspace{.45in}
				\subcaption{Experiment 1}
				\label{fig:payoff}
			\end{subfigure}
		\end{minipage}
		\hspace{-10mm}
		\begin{minipage}{.32\textwidth}
			\centering
			\begin{subfigure}{\textwidth}
				\includegraphics[width=\textwidth]{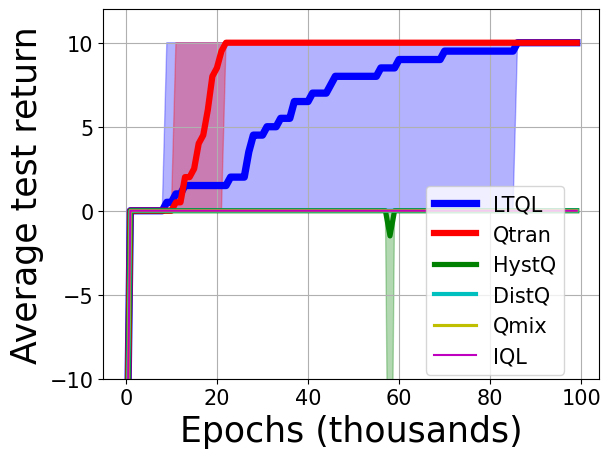}
				\subcaption{Experiment 2}
				\label{fig:tmdp_results}
			\end{subfigure}
		\end{minipage}
		\hspace{10mm}
		\begin{minipage}{.32\textwidth}
			\centering
			\begin{subfigure}{\textwidth}
				\includegraphics[width=\textwidth]{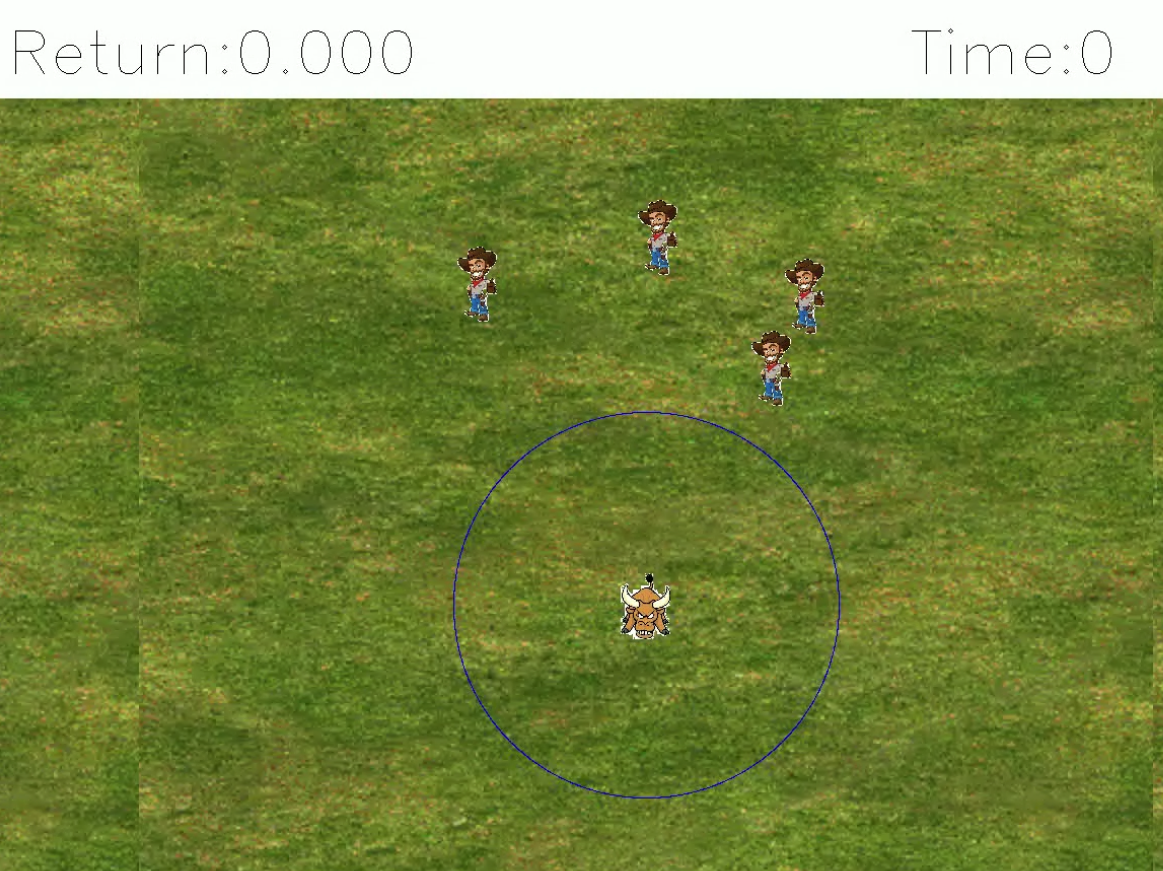}
				\subcaption{Experiment 3}
				\label{fig:cowboy_game}
			\end{subfigure}
		\end{minipage}
		\begin{minipage}{.32\textwidth}
			\centering
			\begin{subfigure}{\textwidth}
				\includegraphics[width=\textwidth]{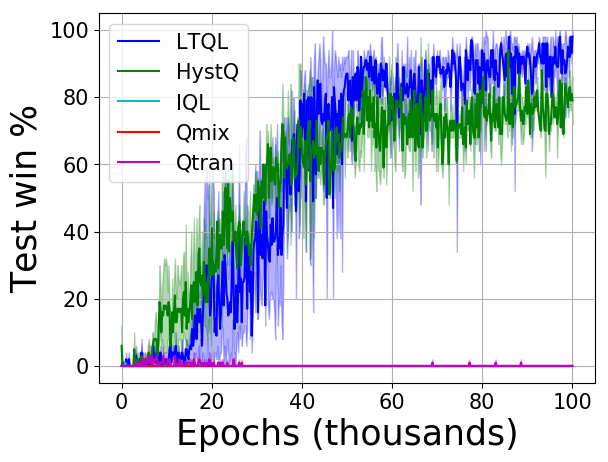}
				\subcaption{Experiment 3}
				\label{fig:cowboy_bull_test_win_rate}
			\end{subfigure}
		\end{minipage}
		\hspace{10mm}
		\begin{minipage}{.32\textwidth}
			\centering
			\begin{subfigure}{\textwidth}
				\includegraphics[width=\textwidth]{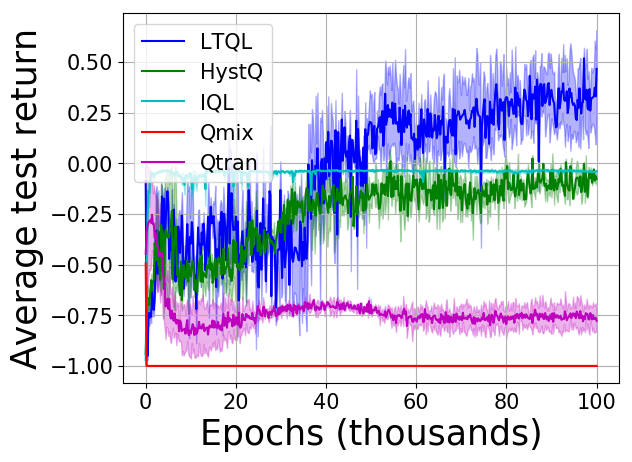}
				\subcaption{Experiment 3}
				\label{fig:cowboy_bull_returns}
			\end{subfigure}
		\end{minipage}
	\end{center}
	\vspace{.3in}
	\caption{The dark curves show the mean over all seeds while the shaded region show the minimum and maximum values. We clarify that in figure \ref{fig:tmdp_results} the curves corresponding to HystQ and DistQ, which are partially occluded, converge to 0.}
	\label{fig:matrix_game}
\end{figure*}
\section{EXPERIMENTS}
\subsection{Matrix Game}\label{subsec:matrix_game}
The first experiment is a simple matrix game (figure \ref{fig:payoff} shows the payoff structure) with multiple team optimum policies to evaluate the resilience of the algorithm to the coordination issue mentioned in section \ref{section:dynamic_programming}. In this case, we implemented IQL, DistQ, LTQL, Qmix and Qtran (we do not include a curve labeled HystQ because in deterministic environments with tabular representation the optimum value for the small step-size is $0$, in which case it becomes exactly equivalent to DistQ). All algorithms are implemented in tabular form.\footnote{In the case of Qmix we used tabular representations for the individual $q$ functions and a NN for the mixing network.} In all cases we used uniform exploratory policies ($\epsilon=1$) and we did not use replay buffers. IQL fails at this task and oscillates due to the perceived time-varying environment (see figure 2 in appendix 6.5). DistQ converges to \eqref{eq:distq_matrix_1}-\eqref{eq:distq_matrix_2}, which clearly shows why DistQ has a coordination issue. However, LTQL converges to either of the two possible solutions \eqref{eq:logical_matrix_1} or \eqref{eq:logical_matrix_2} (depending on the seed) for which individual greedy policies result in team optimal policies. Qmix fails at identifying an optimum team policy and the resulting joint $q$-function obtained using the mixing network also fails at predicting the rewards. Qmix converges to \eqref{eq:qmix_matrix}. The joint $q$-function is shown in appendix 6.5. And Qtran also oscillates due to the fact that in this matrix game there are two jointly optimal actions. In appendix 6.5 we include the learning curves of all algorithms for the readers reference along with a brief discussion.
\begin{align}
&q^1(a^1)=\max_{a^2}q^\dagger(a^1,a^2)=[2,2]\label{eq:distq_matrix_1}\\
&q^2(a^2)=\max_{a^1}q^\dagger(a^1,a^2)=[0,2,2]\label{eq:distq_matrix_2}\\
&q^{1,\star}(a^1)=[2,1]\hspace{4mm}q^{2,\star}(a^2)=[0,2,0]\label{eq:logical_matrix_1}\\
&q^{1,\star}(a^1)=[0,2]\hspace{4mm}q^{2,\star}(a^2)=[0,1,2]\label{eq:logical_matrix_2}\\
&q^{1}(a^1)=[-0.7,1.1]\hspace{4mm}q^{2}(a^2)=[-3.5,1.8,0.6]\label{eq:qmix_matrix}
\end{align}

\subsection{Stochastic Finite Environment}
In this experiment we use a tabular representation in an environment that is stochastic and episodic. The environment is a linear grid with 4 positions and 2 agents. At the beginning of the episode, the agents are initialized in the far right. Agent 1 cannot move and has 2 actions (\textit{push button} or \textit{not push}), while agent 2 has 3 actions (\textit{stay}, \textit{move left} or \textit{move right}). If agent 2 is located in the far left and chooses to \textit{stay} while agent 2 chooses \textit{push}, the team receives a $+10$ reward. If the button is pushed while agent 2 is moving left the team receives a $-30$ reward. This negative reward is also obtained if agent 2 stays still in the leftmost position and agent 1 does not push the button. All rewards are subject to additive Gaussian noise with mean 0 and standard deviation equal to 1. Furthermore if agent 2 tries to move beyond an edge (left or right), it stays in place and the team receives a Gaussian reward with $0$ mean and standard deviation equal to $3$. The episode finishes after 5 timesteps or if the team gets the $+10$ reward (whichever happens first). We ran the simulation $20$ times with different seeds. Figure \ref{fig:tmdp_results} shows the average test return\footnote{The average test return is the return following a greedy policy averaged over $50$ games.} (without the added noise) of IQL, LTQL, HystQ, DistQ, Qmix and Qtran. As can be seen, LTQL and Qtran are the only algorithms capable of learning optimal team policies. In appendix 6.6 we specify the hyperparameters and include the learning curves of the $Q$-functions along with a discussion on the performance of each algorithm.

\subsection{Cowboy Bull Game}\label{subsec:cowboy_bull}
In this experiment we use a more complex environment, a challenging predator-prey type game with partial observability, in which 4 cowboys try to catch a bull (see figure \ref{fig:cowboy_game}). The position of all players is a continuous variable (and hence the state space is continuous). The space is unbounded and the bull can move $20\%$ faster than the cowboys. The bull follows a fixed stochastic policy, which is handcrafted to mimic natural behavior and evade capture. Due to the unbounded space and the fact that the bull moves faster than the cowboys, it cannot be captured unless all agents develop a coordinated strategy (the bull can only be caught if the agents first surround it and then evenly close in). The task is episodic and ends after 75 timesteps or when the bull is caught. Each agent has 5 actions (the four moves plus \textit{stay}). When the bull is caught a $+1$ reward is obtained and the team also receives a small penalty ($-1/(4\hspace{-0.3mm}\times\hspace{-0.4mm}75)$) for every agent that moves. Note that due to the reward structure there is a very easily attainable Nash equilibrium, which is for every agent to stay still (since in this way they do not incur in the penalties associated with movement). Figure \ref{fig:cowboy_bull_test_win_rate} shows the test win percentage\footnote{Percentage of games, out of 50, in which the team succeeds to catch the bull following a greedy policy.} and figure \ref{fig:cowboy_bull_returns} shows the average test return for IQL, LTQL, HystQ, Qmix and Qtran. The best performing algorithm is LTQL. HystQ learns a policy that catches the bull $80\%$ of the times, although it fails at obtaining returns higher than zero. IQL fails because the agents quickly converge to the policy of never moving (to avoid incurring in the negative rewards associated with movement). We believe that the poor performance of Qmix in this task is a consequence of its limited representation capacity due to its monotonic factoring model. Qtran fails in this complex scenario, which is in agreement with results reported in \cite{rashid2020monotonic} where Qtran also shows poor performance in the \textit{Starcraft multi-agent challenge} (SMAC) \cite{smac}. In the appendix we provide all hyperparameters and implementation details, we detail the bull's policy and the observation function. All code\footnote{Code is also available at https://github.com/lcassano/Logical-Team-Q-Learning-paper.}, a pre-trained model and a video of the policy learned by LTQL are included as supplementary material.

\section{CONCLUDING REMARKS}
In this article we have introduced \textit{Logical Team Q-Learning}, which has the $5$ desirable properties mentioned in the introduction. LTQL does not impose constraints on the learned individual $Q$-functions and hence it can solve environments where state of the art algorithms like Qmix and Qtran fail. The algorithm fits in the centralized training and decentralized execution paradigm. It can also be implemented in a fully distributed manner in situations where all agents have access to each others' observations and actions.

\bibliography{refs}

\newpage
\onecolumn
\section*{Appendix}
\subsection{Proof of proposition \ref{proposition:nash}}\label{app:nash}
Consider the matrix game with two agents, each of which has two actions ($\mathcal{A}=\{\alpha;\beta\}$) and the following reward structure:
\begin{table}[H]
	\setlength{\extrarowheight}{2pt}
	\centering
	\caption*{Reward structure}
	\begin{tabular}{*{4}{c|}}
		\multicolumn{2}{c}{} & \multicolumn{2}{c}{Agent $2$}\\\cline{3-4}
		\multicolumn{1}{c}{} &  & $\alpha$  & $\beta$ \\\cline{2-4}
		\multirow{2}*{Agent $1$}  & $\alpha$ & $0$ & $-1$ \\\cline{2-4}
		& $\beta$ & $-1$ & $1$ \\\cline{2-4}
	\end{tabular}
\end{table}

For this case $q^\dagger$, $\pi^\dagger$, $q^{1,\star}$, $q^{2,\star}$, $\pi^{1,\star}$ and $\pi^{2,\star}$ are given by:
\begin{table}[H]
	\begin{minipage}{.2\linewidth}
		\setlength{\extrarowheight}{2pt}
		\centering
		\caption*{$q^\dagger(a^1,a^2)$}
		\begin{tabular}{|c|c|c|}
			\cline{1-3}
			& $\alpha$  & $\beta$ \\\cline{1-3}
			$\alpha$ & $0$ & $-1$ \\\cline{1-3}
			$\beta$ & $-1$ & $1$ \\\cline{1-3}
		\end{tabular}
	\end{minipage}%
	\begin{minipage}{.2\linewidth}
		\setlength{\extrarowheight}{2pt}
		\centering
		\caption*{$\pi^\dagger(a^1,a^2)$}
		\begin{tabular}{|c|c|c|}
			\cline{1-3}
			& $\alpha$  & $\beta$ \\\cline{1-3}
			$\alpha$ & $0$ & $0$ \\\cline{1-3}
			$\beta$ & $0$ & $1$ \\\cline{1-3}
		\end{tabular}
	\end{minipage}
	\begin{minipage}{.14\linewidth}
		\setlength{\extrarowheight}{2pt}
		\centering
		\caption*{$q^{1,\star}(a)$}
		\begin{tabular}{|c|c|}
			\cline{1-2}
			$\alpha$  & $\beta$ \\\cline{1-2}
			$-1$ & $1$ \\\cline{1-2}
		\end{tabular}
	\end{minipage}
	\begin{minipage}{.14\linewidth}
		\setlength{\extrarowheight}{2pt}
		\centering
		\caption*{$q^{2,\star}(a)$}
		\begin{tabular}{|c|c|}
			\cline{1-2}
			$\alpha$  & $\beta$ \\\cline{1-2}
			$-1$ & $1$ \\\cline{1-2}
		\end{tabular}
	\end{minipage}
	\begin{minipage}{.14\linewidth}
		\setlength{\extrarowheight}{2pt}
		\centering
		\caption*{$\pi^{1,\star}(a)$}
		\begin{tabular}{|c|c|}
			\cline{1-2}
			$\alpha$  & $\beta$ \\\cline{1-2}
			$0$ & $1$ \\\cline{1-2}
		\end{tabular}
	\end{minipage}
	\begin{minipage}{.14\linewidth}
		\setlength{\extrarowheight}{2pt}
		\centering
		\caption*{$\pi^{2,\star}(a)$}
		\begin{tabular}{|c|c|}
			\cline{1-2}
			$\alpha$  & $\beta$ \\\cline{1-2}
			$0$ & $1$ \\\cline{1-2}
		\end{tabular}
	\end{minipage}
\end{table}
Notice that as expected, $q^\star$ satisfies \eqref{eq:optimal_fixed_point_1}. However, note that \eqref{eq:optimal_fixed_point_1} is also satisfied by the following $q$ function which is different from $q^\star$.

\begin{align}
&q(a=\alpha)=0,\hspace{5mm}q(a=\beta)=-1
\end{align}
Notice further that the team policy obtained by choosing actions in a greedy fashion with respect to $q$ constitutes a sub-optimal Nash equilibrium. 

\subsection{Proof of Lemma \ref{lemma:over_estimate}}\label{app:over_estimate}
We start defining:
\begin{align}
&q_U=\max\{r_{\textrm{max}}(1-\gamma)^{-1},\max_{k,s,a^k}q^k(s,a^k)\}\\
&q_U^k(s,a^k)=q_U
\end{align}
where $r_{\max}=\max_{s,\bar{a}}r(s,\bar{a})$. We recall that to simplify notation we defined $\Ex_{\mathcal{P},f}\bs{r}(s,a^k,a^{-k},\bs{s'})=r(s,a^k,a^{-k})$. The first part of the proof consists in upper bounding any sequence of the form $\mathcal{B}_{I}^{K_\ell}\mathcal{B}_{E}^{K_{n-1}}\cdots\mathcal{B}_E^{K_1}\mathcal{B}_I^{K_0}q_U^{k}(s,a^k)$, where $K_\ell\in\mathbb{N}$ for all $\ell$. Applying operator $\mathcal{B}_{I}$ to $q_U^{k}(s,a^k)$ we get:
\begin{align}
\mathcal{B}_{I}q_U^{k}(s,a^k)&=\max\big\{q_U,\max_{a^{-k}}r(s,a^k,a^{-k})+\gamma q_U\big\}=q_U
\end{align}
Therefore, $\mathcal{B}_{I}^{K_0}q_U^{k}(s,a^k)=q_U^{k}(s,a^k)$ for any $K_0\in\mathbb{N}$. Applying operator $\mathcal{B}_{E}$ we get:
\begin{align}
\mathcal{B}_{E}q_U^{k}(s,a^k)&=\Ex_{\bs{s'}\sim\mathcal{P}}\big(r(s,a^k,a^{-k})+\gamma q_{U}\big)\big|_{a^{n}=\argmax\limits_{a^n}q_U\hspace{2mm}\forall n\neq k}\leq\max_{a^{-k}}r(s,a^k,a^{-k})+\gamma q_{U}\\
\mathcal{B}_{E}^2q_U^{k}(s,a^k)&\leq\Ex_{\bs{s'}\sim\mathcal{P}}\big(r(s,a^k,a^{-k})\hspace{-0.7mm}+\hspace{-0.7mm}\gamma\max_{\bar{a}'}r(\bs{s'},\bar{a}')\hspace{-0.7mm}+\hspace{-0.7mm}\gamma^2 q_{U}\big)\big|_{a^{n}=\argmax\limits_{a^n}\mathcal{B}_{E}q_U^{n}(s,a^n)\hspace{2mm}\forall n\neq k}\nonumber\\
&\leq\max_{a^{-k}}\Ex_{\bs{s'}\sim\mathcal{P}}\left(r(s,a^k,a^{-k})+\gamma\max_{\bar{a}'}r(\bs{s'},\bar{a}')+\gamma^2q_{U}\right)\\
\mathcal{B}_{E}^{K_1}q_U^{k}(s,a^k)&\leq\max_{a_0^{-k},\bar{a}_1,\cdots,\bar{a}_{K_1-1}}\Ex\left(\sum_{i=0}^{K_1-1}\gamma^ir(\bs{s}_i,a_i^k,a_i^{-k})|\bs{s}_0=s\right)+\gamma^{K_1}q_{U}
\end{align}
Further application of $\mathcal{B}_{I}$ we get:
\begin{align}
\mathcal{B}_{I}\mathcal{B}_{E}^{K_1}q_U^{k}(s,a^k)&\leq\max\bigg\{\max_{a_0^{-k},\bar{a}_1,\cdots,\bar{a}_{K_1-1}}\Ex\left(\sum_{i=0}^{K_1-1}\gamma^ir(\bs{s}_i,a_i^k,a_i^{-k})|\bs{s}_0=s\right)+\gamma^{K_1}q_{U},\nonumber\\
&\hspace{13mm}\max_{a_0^{-k},\bar{a}_1,\cdots,\bar{a}_{K_1}}\Ex\left(\sum_{i=0}^{K_1}\gamma^ir(\bs{s}_i,a_i^k,a_i^{-k})|\bs{s}_0=s\right)+\gamma^{K_1+1}q_{U}\bigg\}\nonumber\\
&=\max_{a_0^{-k},\bar{a}_1,\cdots,\bar{a}_{K_1-1}}\Ex\left(\sum_{i=0}^{K_1-1}\gamma^ir(\bs{s}_i,a_i^k,a_i^{-k})|\bs{s}_0=s\right)+\gamma^{K_1}q_{U}\label{eq:p1}
\end{align}
Therefore, we conclude that $\mathcal{B}_{I}^{K_2}\mathcal{B}_{E}^{K_1}\mathcal{B}_{I}^{K_0}q_U^{k}(s,a^k)=\mathcal{B}_{E}^{K_1}q_U^{k}(s,a^k)$. More generally, we can write:
\begin{align}
\mathcal{B}_{I}^{K_\ell}\cdots\mathcal{B}_E^{K_1}\mathcal{B}_I^{K_0}q_U^{k}(s,a^k)&=\mathcal{B}_{E}^nq_U^{k}(s,a^k)\leq\max_{a_0^k,\bar{a}_1,\cdots,\bar{a}_{n-1}}\Ex\left(\sum_{i=0}^{\bs{n}-1}\gamma^ir(\bs{s}_i,a_i^k,a_i^{-k})|\bs{s}_0=s\right)+\gamma^{\bs{n}} q_{U}\nonumber\\
&\leq\max_{a^{-k}}q^\dagger(s,a^k,a^{-k})+\gamma^{\bs{n}}\big(q_{U}-\min_{s}\max_{\bar{a}}q^\dagger(s,\bar{a})\big)
\end{align}
where we define $\bs{n}$ to be the total number of times that operator $\mathcal{B}_{E}$ is applied. Notice that if operator $\mathcal{B}_p$ is applied $N$ times, $\bs{n}$ is a random variable that follows a binomial distribution with total samples $N$ and probability $p$. Therefore, we get:
\begin{align}
\mathcal{B}_p^Nq_U^{k}(s,a^k)&\leq\max_{a^{-k}}q^\dagger(s,a^k,a^{-k})+\gamma^{\bs{n}}\big(q_{U}-\min_{s}\max_{\bar{a}}q^\dagger(s,\bar{a})\big)
\end{align}
To ensure that $\mathcal{B}_p^Nq_U^{k}(s,a^k)\in\mathcal{C}_{\delta_1}^U$ we need:
\begin{align}
	&\delta_1\geq \gamma^{\bs{n}}\big(q_{U}-\min_{s}\max_{\bar{a}}q^\dagger(s,\bar{a})\big)\hspace{5mm}\rightarrow\hspace{5mm}{\bs{n}}\geq \log_\gamma\left(\frac{\delta_1}{q_{U}-\min_{s}\max_{\bar{a}}q^\dagger(s,\bar{a})}\right)
\end{align}
Since $n$ follows a binomial distribution we get:
\begin{align}
	&\mathbb{P}\Bigg({\bs{n}}\geq \underbrace{\floor[\bigg]{\log_\gamma\left(\frac{\delta_1}{q_{U}-\min_{s}\max_{\bar{a}}q^\dagger(s,\bar{a})}\right)}}_{\define n_o}\Bigg)=1-\sum_{n=0}^{n_o}{{N}\choose{n}}p^n(1-p)^{N-n}
\end{align}
Therefore we can conclude that:
\begin{align}
&\mathbb{P}\big(\mathcal{B}_p^Nq_U^{k}(s,a^k)\in\mathcal{C}_{\delta_1}^U\big)\geq1-\sum_{n=0}^{n_o}{{N}\choose{n}}p^n(1-p)^{N-n}\\
&\mathcal{C}_{\delta_1}^U=\big\{q^{k}|q^{k}(s,a^k)\leq\max_{a^{-k}}q^\dagger(s,a^k,a^{-k})+\delta_1\hspace{5mm}\forall (k,s,a^k)\hspace{-1mm}\in\hspace{-1mm}(\mathcal{K},\mathcal{S},\mathcal{A}^{k})\big\}
\end{align}
Noting that by construction $\mathcal{B}_p^N q_U^k(s,a^k)\geq\mathcal{B}_p^Nq^k(s,a^k)$ for all $N\geq1$, we get that: 
\begin{align}
&\mathbb{P}\big(\mathcal{B}_p^Nq^{k}(s,a^k)\in\mathcal{C}_{\delta_1}^U\big)\geq1-\sum_{n=0}^{n_o}{{N}\choose{n}}p^n(1-p)^{N-n}
\end{align}
For the special case where $n_o<pN$ we can bound the cumulative distribution function of the binomial distribution using Hoeffding's bound:
\begin{align}
&\mathbb{P}\big(\mathcal{B}_p^Nq^{k}(s,a^k)\in\mathcal{C}_{\delta_1}^U\big)\geq 1-e^{-2N\left(p-\frac{n_o}{N}\right)^2}
\end{align}
which concludes the proof. 

\subsection{Proof of Lemma \ref{lemma:approach}}\label{app:lemma_lower_bound}
We start stating the following auxiliary lemma.
\begin{lemma}\label{lemma:c0u}
	If $q^{k}(s,a^k)\in\mathcal{C}_0^U$ then it holds that $\mathcal{B}_p^Nq^{k}(s,a^k)\in\mathcal{C}_0^U$ for all $N\geq0$.
\end{lemma}
\begin{proof}
	We start noting that if $q^{k}(s,a^k)\in\mathcal{C}_0^U$ then $\mathcal{B}_Eq^{k}(s,a^k)\in\mathcal{C}_0^U$ and $\mathcal{B}_Iq^{k}(s,a^k)\in\mathcal{C}_0^U$.
	\begin{align}
	\mathcal{B}_{I}q^{k}(s,a^k)&=\max\big\{q^{k}(s,a^k),\max_{a^{-k}}\Ex(r(s,a^k,a^{-k},\bs{s'})+\gamma \max_{a^{\prime,k}}q^k(\bs{s'},a^{\prime,k}))\big\}\nonumber\\
	&\stackrel{(a)}{\leq}\max\big\{\max_{a^{-k}}q^{\dagger}(s,a^k),\max_{a^{-k}}\Ex(r(s,a^k,a^{-k},\bs{s'})+\gamma \max_{\bar{a}^{\prime}}q^\dagger(\bs{s'},\bar{a}^{\prime}))\big\}=\max_{a^{-k}}q^{\dagger}(s,a^k)\\
	\mathcal{B}_{E}q^{k}(s,a^k)&=\Ex\big(\bs{r}(s,a^k,a^{-k},\bs{s'})+\gamma\max_{a^{\prime}}q^{k}(\bs{s'},a^{\prime})\big)\big|_{a^{n}=\argmax\limits_{a^n}q^{n}(s,a^n)\hspace{1mm}\forall n\neq k}\nonumber\\
	&\leq\max_{a^{-k}}\Ex\big(\bs{r}(s,a^k,a^{-k},\bs{s'})+\gamma\max_{a^{\prime}}q^{k}(\bs{s'},a^{\prime})\big)\nonumber\\
	&\stackrel{(b)}{\leq}\max_{a^{-k}}\Ex\big(\bs{r}(s,a^k,a^{-k},\bs{s'})+\gamma\max_{\bar{a}^{\prime}}q^{\dagger}(\bs{s'},\bar{a}^{\prime})\big)=\max_{a^{-k}}q^{\dagger}(s,a^k)
	\end{align}
	where in $(a)$ and $(b)$ we used the fact that $q^{k}(s,a^k)\in\mathcal{C}_0^U$. Since it holds that $\mathcal{B}_Eq^{k}(s,a^k)\in\mathcal{C}_0^U$ and $\mathcal{B}_Iq^{k}(s,a^k)\in\mathcal{C}_0^U$ it immediately follows that $\mathcal{B}_p^Nq^{k}(s,a^k)\in\mathcal{C}_0^U$ for all $N\geq0$.
\end{proof}
We follow by noting that applying operator $\mathcal{B}_I$ $L$ times to $q^{k}(s,a^k)<\max_{a^{-k}}q^\dagger(s,a^k,a^{-k})$ we get:
\begin{align}
\mathcal{B}_{I}q^{k}(s,a^k)&=\max\big\{q^{k}(s,a^k),\max_{a^{-k}}\Ex(r(s,a^k,a^{-k},\bs{s'})+\gamma \max_{a^{\prime,k}}q^k(\bs{s'},a^{\prime,k}))\big\}\nonumber\\
&\geq\max_{a^{-k}}\Ex\big(r(s,a^k,a^{-k},\bs{s'})+\gamma \max_{a^{\prime,k}}q^k(\bs{s'},a^{\prime,k})\big)\\
\mathcal{B}_{I}^{L}q^{k}(s,a^k)&=\hspace{-3mm}\max_{a_0^{-k},\bar{a}_1,\cdots,\bar{a}_{L-1}}\hspace{-3mm}\Ex\left(\sum_{i=0}^{L-1}\gamma^ir(\bs{s}_i,a_i^k,a_i^{-k})+\gamma^{L}\max_{a_L^{k}}q^k(\bs{s}_L,a_L^{k})|\bs{s}_0=s,a_0^k=a^k\right)\nonumber\\
\big|\mathcal{B}_{I}^Lq^{k}(s,a^k)-&\max_{a^{-k}}q^\dagger(s,a^k,a^{-k})\big|\stackrel{(c)}{\leq}\gamma^{L}\max_{s}\big|\max_{a^k}q^{k}(s,a^k)-\max_{\bar{a}}q^\dagger(s,\bar{a})\big|\define\epsilon(L)\label{eq:epsilon_L}
\end{align}
where in $(c)$ we used $\mathcal{B}_{I}^{L}q^{k}(s,a^k)<\max_{a^{-k}}q^\dagger(s,a^k,a^{-k})$ from lemma \ref{lemma:c0u}. If $\epsilon(L)<\delta_2$, we get:
\begin{align}
	&\mathcal{B}_{I}^Lq^{k}(s,a^k)\in\mathcal{C}_{\delta_2}\label{eq:epsilon_condition}\\
	&\mathcal{C}_{\delta_2}=\big\{q^{k}|q^{k,\star}(s,a^k)-\delta_2\leq q^{k}(s,a^k)\leq\max_{a^{-k}}q^\dagger(s,a^k,a^{-k})+\delta_2\forall (k,s,a^k)\hspace{-1mm}\in\hspace{-1mm}(\mathcal{K},\mathcal{S},\mathcal{A}^{k})\big\}
\end{align}
\begin{lemma}\label{lemma:c_delta2}
	If $q^{k}(s,a^k)\in\mathcal{C}_{\delta_2}$ and $\delta_2$ is small enough, then it holds that $\mathcal{B}_p^Nq^{k}(s,a^k)\in\mathcal{C}_{\gamma^N\delta_2}$ for all $N>0$.
\end{lemma}
\begin{proof}
	Assume $\delta_2$ satisfies the following relation:
	\begin{align}\label{eq:delta2_cond}
		\delta_2<2^{-1}\min_{s}\big(\max_{\bar{a}}q^\dagger(s,\bar{a})-\max_{\bar{a}\neq\argmax\limits_{\bar{a}} q^\dagger(s,\bar{a})}q^\dagger(s,\bar{a})\big)
	\end{align}
	We clarify that the term in between parenthesis in the r.h.s. of relation \eqref{eq:delta2_cond} is the difference between the optimal $q$-value and the second highest $q$-value for state $s$. Note that if $q^{k}(s,a^k)\in\mathcal{C}_{\delta_2}$ then it trivially follows that $\mathcal{B}_Iq^{k}(s,a^k)\in\mathcal{C}_{\delta_2}$, for the case of $\mathcal{B}_E$ we get:
	\begin{align}\label{eq:be_opt}
		\mathcal{B}_{E}q^{k}(s,a^k)&=\Ex\big(\bs{r}(s,a^k,a^{-k},\bs{s'})+\gamma\max_{a^{\prime}}q^{k}(\bs{s'},a^{\prime})\big)\big|_{a^{n}=\argmax\limits_{a^n}q^{n}(s,a^n)\hspace{1mm}\forall n\neq k}
	\end{align}
	Using equation \eqref{eq:factored_q_1} and the fact that $q^{k}(s,a^k)\in\mathcal{C}_{\delta_2}$ it follows:
	\begin{align}
	&\max_{a^k}q^{k,\star}(s,\bar{a})-\delta_2\leq \max_{a^k}q^{k}(s,a^k)\\
	&q^{k}(s,a^{k,\bullet})\leq\max_{a^{-k}}q^\dagger(s,a^{k,\bullet},a^{-k})+\delta_2\stackrel{(d)}{<}\max_{\bar{a}}q^\dagger(s,\bar{a})-\delta_2=\max_{a^k}q^{k,\star}(s,\bar{a})-\delta_2\\
	&a^{k,\bullet}=\argmax_{a^k\neq \argmax\limits_{a^k}q^{k}(s,a^k)}q^{k}(s,a^k)\label{eq:correct_act}
	\end{align}
	where in $(d)$ we used condition \eqref{eq:delta2_cond}. Combining equations \eqref{eq:be_opt} through \eqref{eq:correct_act} we get:
	\begin{align}
	\mathcal{B}_{E}&q^{k}(s,a^k)=\Ex\big(\bs{r}(s,a^k,a^{-k},\bs{s'})+\gamma\max_{a^{\prime}}q^{k}(\bs{s'},a^{\prime})\big)\big|_{a^{n}=\argmax\limits_{a^n}q^{n,\star}(s,a^n)\hspace{1mm}\forall n\neq k}\nonumber\\
	&=\Ex\big(\bs{r}(s,a^k,a^{-k},\bs{s'})+\gamma\max_{a^{\prime}}q^{k}(\bs{s'},a^{\prime})+\gamma\max_{a^{\prime}}q^{k,\star}(\bs{s'},a^{\prime})-\gamma\max_{a^{\prime}}q^{k,\star}(\bs{s'},a^{\prime})\big)\big|_{a^{n}=\argmax\limits_{a^n}q^{n,\star}(s,a^n)\hspace{1mm}\forall n\neq k}\nonumber\\
	&=q^{k,\star}(s,a^k)+\gamma\Ex\big(\max_{a^{\prime}}q^{k}(\bs{s'},a^{\prime})-\max_{a^{\prime}}q^{k,\star}(\bs{s'},a^{\prime})\big)\big|_{a^{n}=\argmax\limits_{a^n}q^{n,\star}(s,a^n)\hspace{1mm}\forall n\neq k}\label{eq:be_bon}
	\end{align}
	Combining equation \eqref{eq:be_bon} with the fact that $q^{k}(s,a^k)\in\mathcal{C}_{\delta_2}$ we get:
	\begin{align}
	&q^{k,\star}(s,a^k)-\gamma\delta_2\leq\mathcal{B}_{E}q^{k}(s,a^k)\leq\max_{a^{-k}}q^\dagger(s,a^k,a^{-k})-\gamma\delta_2
	\end{align}
	which completes the proof.
\end{proof}
\begin{lemma}\label{lemma:c_delta2_condition}
	For any $q^{k}(s,a^k)\in\mathcal{C}_{0}^U$, $\delta_2>0$ and $N\geq$, it holds that $\mathcal{B}_p^Nq^{k}(s,a^k)\in\mathcal{C}_{\delta_2}$ as long as the sequence of $N$ operators $\mathcal{B}_p$ includes at least $L$ consecutive $\mathcal{B}_I$'s. Where $L$ is given by:
	\begin{align}\label{eq:L}
		L=\ceil[\Bigg]{\log_\gamma\left(\frac{\delta_2}{\max_{s}\big|\max_{a^k}q^{k}(s,a^k)-\max_{\bar{a}}q^\dagger(s,\bar{a})\big|}\right)}
	\end{align}
\end{lemma}
\begin{proof}
	The statement is an immediate consequence of lemmas \ref{lemma:c0u} and \ref{lemma:c_delta2} and relation \eqref{eq:epsilon_L}. Relation \eqref{eq:L} follows from combining equation \eqref{eq:epsilon_L} and $\epsilon(L)<\delta_2$.
\end{proof}
Notice that the probability of applying operator $\mathcal{B}_I$ at least $L$ consecutive times when operator $\mathcal{B}_p$ is applied $N\geq L$ times, is the same as the probability of obtaining at least $L$ consecutive heads when a biased coin (with probability of head $1-p$) is tossed $N$ times. This problem has been extensively studied and the result is available in the literature. We state the following useful result from \cite{uspensky}:
\begin{lemma}\label{lemma:uspensky}
	\cite{uspensky}: If a biased coin (with probability of head being $1-p$) is tossed $N\geq L$ times, the probability of having a sequence of at least $L$ consecutive heads is given by:
	\begin{align}
		\mathbb{P}(L)&=1-\beta_{N,L}+(1-p)^L\beta_{N-L,L}\\
		\beta_{N,L}&=\sum_{j=0}^{\floor{N/(L+1)}}(-1)^j{{N-jL}\choose{j}}\big(p(1-p)^L\big)^j
	\end{align}
	Furthermore, if $p>0.5$ the probability can be lower bounded as follows:
	\begin{align}\label{eq:uspensky_prob}
	\mathbb{P}(L)&\geq 1-\frac{1-(1-p)\xi_1}{p\xi_1(1+L-L\xi_1)}\xi_1^{-N}-\frac{L}{p}(1-p)^{N+2}
	\end{align}
	where $1<\xi_1<1+L^{-1}$.
\end{lemma}
Combining lemmas \ref{lemma:uspensky} and \ref{lemma:c_delta2_condition} we can conclude that after $N\geq L>0$ applications of operator $\mathcal{B}_p$ to any set of $K$ $q^{k}(s,a^k)\in\mathcal{C}_0^U$ functions it holds:
\begin{align}
&\mathbb{P}\big(\mathcal{B}_p^Nq^{k}(s,a^k)\in\mathcal{C}_{\delta_2}\big)\geq1-\beta_{N,L}+(1-p)^L\beta_{N-L,L}\label{eq:prob}\\
&L=\ceil[\Bigg]{\log_\gamma\left(\frac{\delta_2}{\max_{s}\big|\max_{a^k}q^{k}(s,a^k)-\max_{\bar{a}}q^\dagger(s,\bar{a})\big|}\right)}\\
&\mathcal{C}_{\delta_2}=\big\{q^{k}|q^{k,\star}(s,a^k)-\delta_2\leq q^{k}(s,a^k)\leq\max_{a^{-k}}q^\dagger(s,a^k,a^{-k})+\delta_2\forall (k,s,a^k)\hspace{-1mm}\in\hspace{-1mm}(\mathcal{K},\mathcal{S},\mathcal{A}^{k})\big\}
\end{align}
If $p>0.5$ we can lower bound probability \eqref{eq:prob} by:
\begin{align}
&\mathbb{P}\big(\mathcal{B}_p^Nq^{k}(s,a^k)\in\mathcal{C}_{\delta_2}\big)\geq 1-\frac{1-(1-p)\xi_1}{p\xi_1(1+L-L\xi_1)}\xi_1^{-N}-\frac{L}{p}(1-p)^{N+2}
\end{align}

\subsection{Tabular \textit{Logical Team Q-Learning}}\label{app:tabular_version}
In the particular case where the MDP is deterministic (and hence $\Ex(\bs{r}(s,\bar{a},\bs{s'})+\gamma\max_{a^{\prime}}q^{k}(s',a^{\prime}))=r(s,\bar{a},s')+\gamma\max_{a^{\prime}}q^{k}(s',a^{\prime})$) the tabular version of \textit{Logical Team Q-learning} is given by algorithm \ref{algorithm:tabular_logical}.
\begin{algorithm}[H]
	\caption{Tabular \textit{Logical Team Q-Learning} for deterministic MDPs}
	\label{algorithm:tabular_logical}
	\begin{algorithmic}
		\STATE{\bfseries Initialize:} an empty replay buffer $\mathcal{R}$ and estimates $\widehat{q}_B^{k}$ and $\widehat{q}_U^{k}$.
		\FOR{iterations $e=0,\ldots,E$}
		\STATE Sample $T$ transitions $(s,\bar{a},r,s')$ by following some behavior policy which guarantees all joint actions are sampled with non-zero probability and store them in $\mathcal{R}$.
		\FOR{iterations $i=0,\ldots,I$}
		\STATE Sample a transition $(s,\bar{a},r,s')$ from $\mathcal{R}$.
		\FOR{agent $k=1,\cdots,K$}
		\IF{$\big(a^{n}=\argmax_{a^n}\widehat{q}^{n}(s,a^n)\hspace{1mm}\forall\hspace{-0.3mm}n\hspace{-0.5mm}\neq\hspace{-0.5mm}k\big)$}
		\STATE $\widehat{q}^{k}(s,a^k)=\widehat{q}^{k}(s,a^k)+\mu\big(r+\max\limits_{a}\widehat{q}^{k}(s',a)-\widehat{q}^{k}(s,a^k)\big)$
		\ELSIF{$\big(r+\max\limits_{a}\widehat{q}^{k}(s',a)>\widehat{q}^{k}(s,a^k)\big)$}
		\STATE $\widehat{q}^{k}(s,a^k)=\widehat{q}^{k}(s,a^k)+\mu\alpha\big(r+\max\limits_{a}\widehat{q}^{k}(s',a)-\widehat{q}^{k}(s,a^k)\big)$
		\ENDIF
		\ENDFOR
		\ENDFOR
		\ENDFOR
	\end{algorithmic}
\end{algorithm}

\begin{algorithm}[H]
	\caption{Tabular \textit{Logical Team Q-Learning}}
	\label{algorithm:stochastic_version}
	\begin{algorithmic}
		\STATE{\bfseries Initialize:} an empty replay buffer $\mathcal{R}$ and estimates $\widehat{q}_B^{k}$ and $\widehat{q}_U^{k}$.
		\FOR{iterations $e=0,\ldots,E$}
		\STATE Sample $T$ transitions $(s,\bar{a},r,s')$ by following some behavior policy and store them in $\mathcal{R}$.
		\FOR{iterations $i=0,\ldots,I$}
		\STATE Sample a transition $(s,\bar{a},r,s')$ from $\mathcal{R}$.
		\FOR{agent $k=1,\cdots,K$}
		\IF{$a^{n}=\argmax_{a^n}\widehat{q}_B^{n}(s,a^n)\hspace{1mm}\forall\hspace{-0.3mm}n\hspace{-0.5mm}\neq\hspace{-0.5mm}k$}
		\STATE $\widehat{q}_B^{k}(s,a^k)=\widehat{q}_B^{k}(s,a^k)+\mu\big(r+\max\limits_{a}\widehat{q}_U^{k}(s',a)-\widehat{q}_B^{k}(s,a^k)\big)$
		\STATE $\widehat{q}_U^{k}(s,a^k)=\widehat{q}_U^{k}(s,a^k)+\mu\big(r+\max\limits_{a}\widehat{q}_U^{k}(s',a)-\widehat{q}_U^{k}(s,a^k)\big)$ 
		\ENDIF
		\IF{$\big(r+\max\limits_{a}\widehat{q}_U^{k}(s',a)>\widehat{q}_B^{k}(s,a^k)\big)$}
		\STATE $\widehat{q}_B^{k}(s,a^k)=\widehat{q}_B^{k}(s,a^k)+\mu\alpha\big(r+\max\limits_{a}\widehat{q}_U^{k}(s',a)-\widehat{q}_B^{k}(s,a^k)\big)$
		\ENDIF
		\ENDFOR
		\ENDFOR
		\ENDFOR
	\end{algorithmic}
\end{algorithm}
If algorithm \ref{algorithm:tabular_logical} were applied to a stochastic MDP, due to condition $c_2$ $\big(r+\max_{a}\widehat{q}^{k}(s',a)>\widehat{q}^{k}(s,a^k)\big)$, it would be subject to bias, which would propagate through bootstrapping and hence could compromise its performance. This can be solved by having a second unbiased estimate $q_U$ that is updated only when $c_1$ is satisfied and use this unbiased estimate to bootstrap. The resulting algorithm is shown in algorithm \ref{algorithm:stochastic_version}.

\subsection{Matrix game additional results}\label{app:experiment1}
We start specifying the hyperparameters. For IQL, DistQ, LTQL and Qtran we used a step-size equal to $0.1$. The $\alpha$ parameter for LTQL is equal to $1$. The mixing network in Qmix has $2$ hidden layers with $5$ units each, the nonlinearity used was the \textit{ELu} and the step-size used was $0.05$ (we had to make it smaller than the others to make the SGD optimizer converge). We finally remark that due to the use of a NN in Qmix we had to train this algorithm with $100$ times more games (notice the x-axis in figure \ref{fig:convergence_curves_matrix_game}).

In figure \ref{fig:convergence_curves_matrix_game} we show the convergence curves for IQL \ref{fig:iql_matrix_game}, DistQ \ref{fig:dist_q}, LTQL \ref{fig:biased_logic_seed_0}-\ref{fig:unbiased_logic_seed_1}, Qmix \ref{fig:q_mix_matrix_game} and Qtran \ref{fig:q_tran_matrix_game}. Figures \ref{fig:biased_logic_seed_0} and \ref{fig:unbiased_logic_seed_0} correspond to the deterministic (algorithm \ref{algorithm:tabular_logical}) and general version (algorithm \ref{algorithm:stochastic_version}) of LTQL, respectively. Figures \ref{fig:biased_logic_seed_1} and \ref{fig:unbiased_logic_seed_1} also show curves for LTQL but use different seeds and show that this algorithm can converge to either of the two following set of factored $q$-functions:
\begin{align}
q^{1,\star}(a^1)&=[2,1]\hspace{4mm}q^{2,\star}(a^2)=[0,2,0]\hspace{8mm}\textbf{or}\hspace{8mm}q^{1,\star}(a^1)=[0,2]\hspace{4mm}q^{2,\star}(a^2)=[0,1,2]
\end{align} 
One interesting fact to note is that the suboptimal values of $q_B^k$ (in figures \ref{fig:biased_logic_seed_0} and \ref{fig:biased_logic_seed_1}) do not converge while the same values do converge in the case of $q_U^k$. The reason for this is that there is a parallel between including the unbiased estimate $q_U^k$ in the RL algorithm and including an application of operator $\mathcal{B}_E$ at the end of the dynamic programming procedure described in theorem 1. The proof of the theorem shows that if such operator is not included, only the optimal values of the estimates generated by \textit{Logical Team Q-learning} converge to $q^{k,\star}$, while the values corresponding to suboptimal actions oscillate in the region between $q^{k,\star}(s,a^k)$ and $\max_{a^{-k}}q^\dagger(s,a^k,a^{-k})$ (which is what happens when the unbiased estimate $q_U^k$ is not used in algorithm \ref{algorithm:tabular_logical}). Note that in the case of algorithm \ref{algorithm:stochastic_version}
where the unbiased estimate $q_U^k$ is included all values $q^{k}$ converges to $q^{k,\star}$ for all actions, not just the optimal ones (this is equivalent to including the operator $\mathcal{B}_E$ after $\mathcal{B}_p^N$ in the dynamic programming setting).
\begin{figure}[h]
	\begin{minipage}{.32\linewidth}
		\setlength{\extrarowheight}{2pt}
		\centering
		\begin{subfigure}{\textwidth}
			\includegraphics[width=\textwidth]{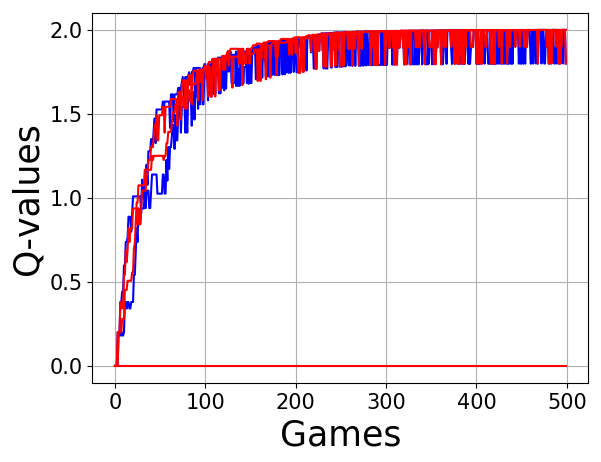}
			\subcaption{IQL}
			\label{fig:iql_matrix_game}
		\end{subfigure}
	\end{minipage}
	\begin{minipage}{.32\linewidth}
		\setlength{\extrarowheight}{2pt}
		\centering
		\begin{subfigure}{\textwidth}
			\includegraphics[width=\textwidth]{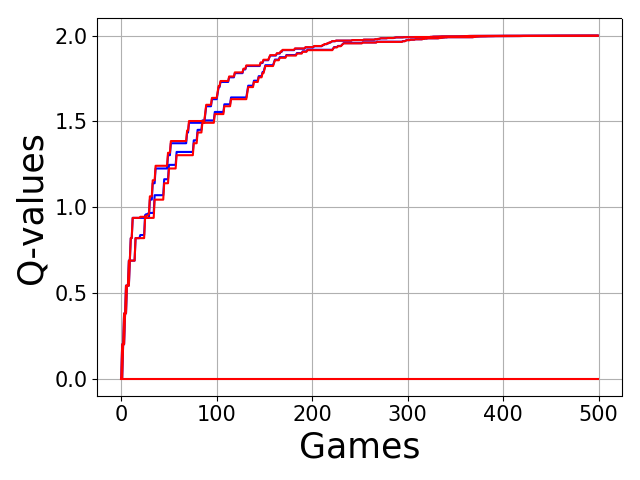}
			\subcaption{DistQ}
			\label{fig:dist_q}
		\end{subfigure}
	\end{minipage}
	\begin{minipage}{.32\linewidth}
		\setlength{\extrarowheight}{2pt}
		\centering
		\begin{subfigure}{\textwidth}
			\includegraphics[width=\textwidth]{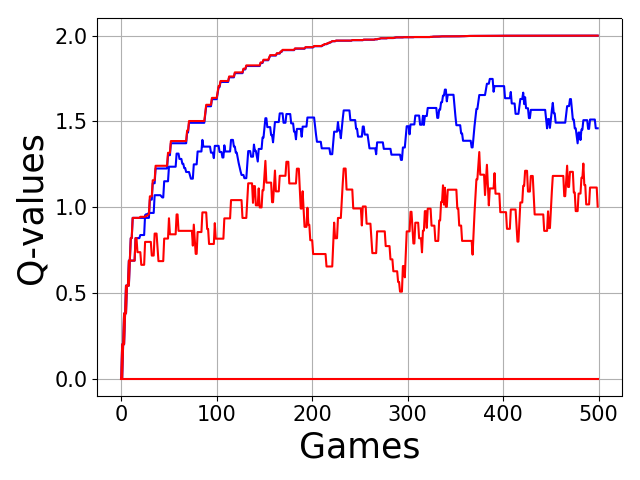}
			\subcaption{$q_B$ (seed=0)}
			\label{fig:biased_logic_seed_0}
		\end{subfigure}
	\end{minipage}
	\begin{minipage}{.32\linewidth}
		\setlength{\extrarowheight}{2pt}
		\centering
		\begin{subfigure}{\textwidth}
			\includegraphics[width=\textwidth]{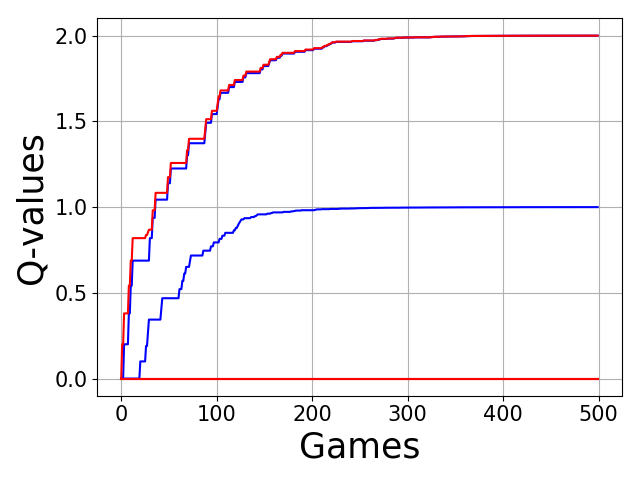}
			\subcaption{$q_U$ (seed=0)}
			\label{fig:unbiased_logic_seed_0}
		\end{subfigure}
	\end{minipage}
	\hspace{3mm}
	\begin{minipage}{.33\linewidth}
		\vspace{-2mm}
		\setlength{\extrarowheight}{2pt}
		\centering
		\begin{subfigure}{\textwidth}
			\includegraphics[width=\textwidth]{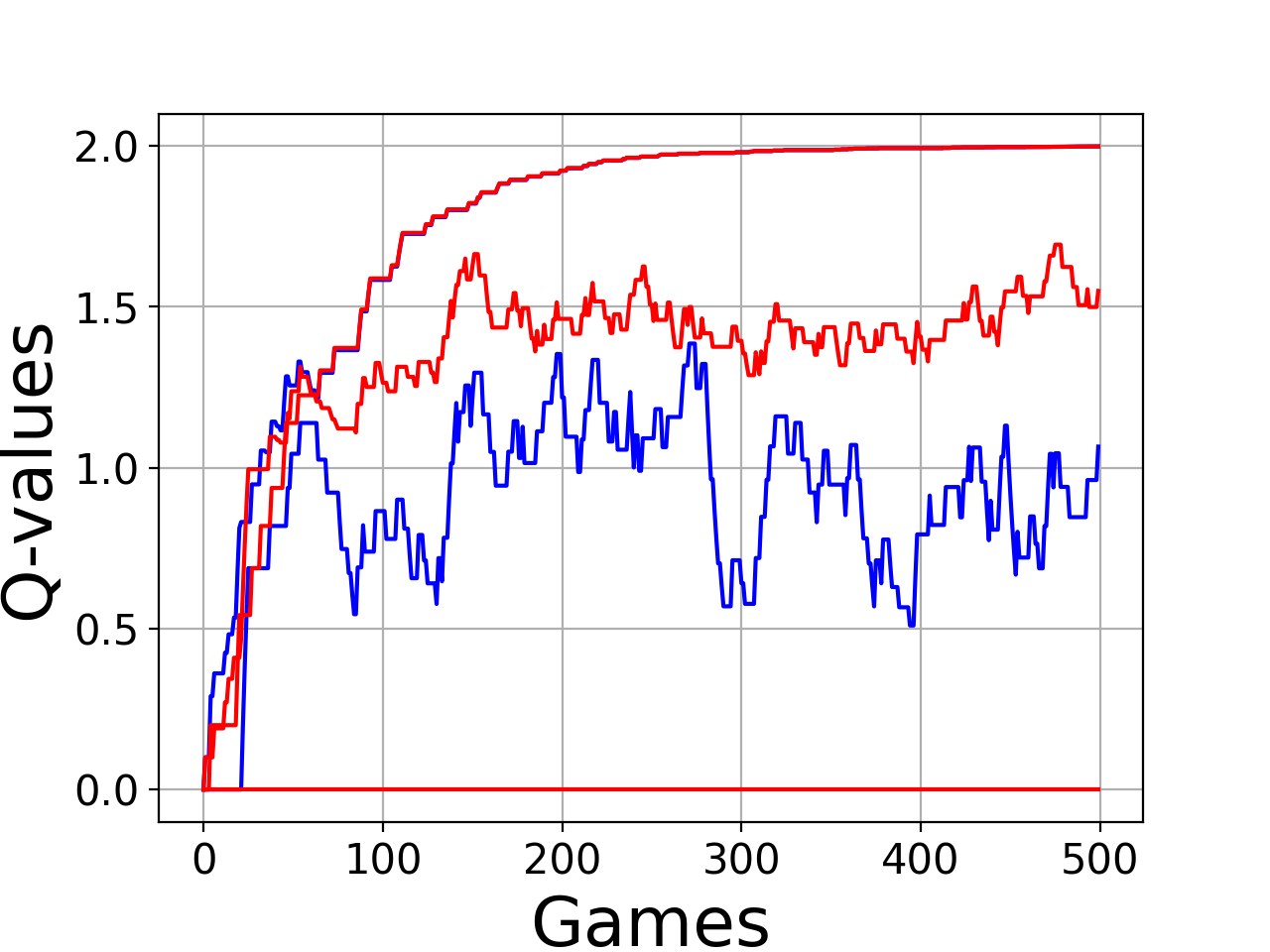}
			\subcaption{$q_B$ (seed=1)}
			\label{fig:biased_logic_seed_1}
		\end{subfigure}
	\end{minipage}
	\hspace{2mm}
	\begin{minipage}{.33\linewidth}
		\vspace{-2mm}
		\setlength{\extrarowheight}{2pt}
		\centering
		\begin{subfigure}{\textwidth}
			\includegraphics[width=\textwidth]{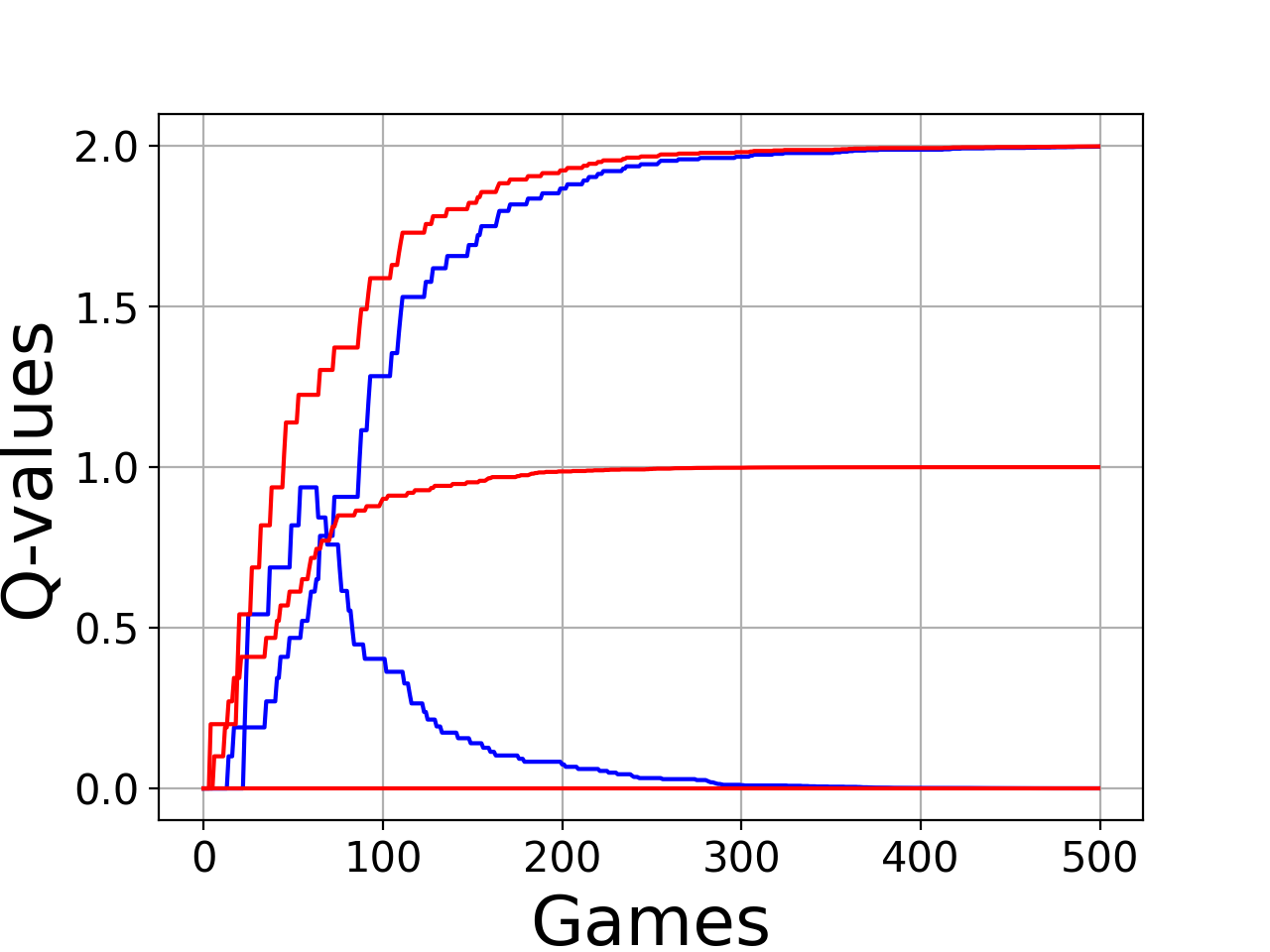}
			\subcaption{$q_U$ (seed=1)}
			\label{fig:unbiased_logic_seed_1}
		\end{subfigure}
	\end{minipage}
	\begin{minipage}{.32\linewidth}
		\setlength{\extrarowheight}{2pt}
		\centering
		\begin{subfigure}{\textwidth}
			\includegraphics[width=\textwidth]{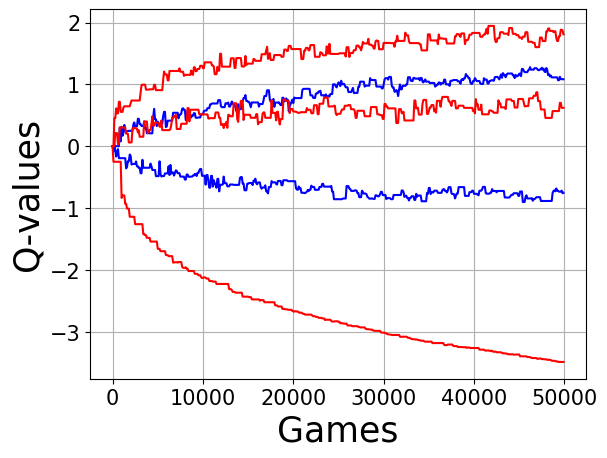}
			\subcaption{Qmix}
			\label{fig:q_mix_matrix_game}
		\end{subfigure}
	\end{minipage}
	\hspace{1mm}
	\begin{minipage}{.32\linewidth}
		\setlength{\extrarowheight}{2pt}
		\centering
		\begin{subfigure}{\textwidth}
			\includegraphics[width=\textwidth]{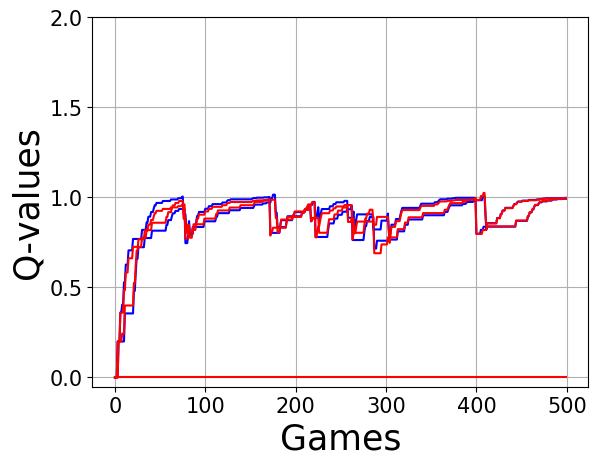}
			\subcaption{Qtran}
			\label{fig:q_tran_matrix_game}
		\end{subfigure}
	\end{minipage}
	\caption{Matrix game. In all figures the red curves correspond to the three actions of agent 2, while the two blue curves correspond to the two actions from agent 1.}
	\label{fig:convergence_curves_matrix_game}
\end{figure}
Below we show the joint $q$ values generated by Qmix's mixing network.
\begin{table}[H]
	\centering
	\begin{tabular}{*{5}{c|}}
		\multicolumn{1}{c}{} & \multicolumn{1}{c}{} &\multicolumn{3}{c}{Agent $2$}\\\cline{3-5}
		\multicolumn{1}{c}{\multirow{3}*{\begin{turn}{-90}Agent 1\end{turn}}}  &  & $a_1$ $(-3.49)$  & $a_2$ $(1.83)$ & $a_3$ $(0.62)$ \\\cline{2-5}
		& $b_1$ $(-0.74)$ & $-4.78\times10^{-2}$ & $1.17$ & $6.86\times10^{-1}$ \\\cline{2-5}
		& $b_2$ $(1.09)$ & $1.51\times10^{-3}$ & $1.57$ & $1.09$\\\cline{2-5}
	\end{tabular}
	\vspace{3mm}
	\caption{Qmix full results}
\end{table}

Code is available at https://github.com/lcassano/Logical-Team-Q-Learning-paper.

\subsection{Stochastic TMDP additional results}\label{app:experiment2}
In this environment agents rely on the following observations. The observation corresponding to agent 1 is a vector with two binary elements: the first one indicates whether or not agent 2 is in the leftmost position, and the second element indicates whether or not there is enough time for agent 2 to reach the leftmost position. The observation corresponding to agent 2 is a vector with two elements: the first one is the number of the position it occupies and the second one is the same as agent 1 (whether there is enough time to reach the leftmost position).

All algorithms are implemented in an on-line manner with no replay buffer. $\epsilon$-greedy exploration with a decaying schedule is used in all cases ($\epsilon=\max[0.05, 1-\textit{epoch}/2\times10^5]$). The step-size used is $\mu=0.025$ and the smaller step-size for HystQ is $\mu_{\textrm{small}}=10^{-2}$, in the case of Qmix we used $\mu=10^{-3}$ to guarantee stability. The $\alpha$ parameter for LTQL is equal to $1$.

Figures \ref{fig:logic_q_values} show the estimated $q$-values for LTQL corresponding to 4 different observations at the 4 positions. Note that in all figures the optimum action has the highest value and correctly estimates the return corresponding to the optimal team policy ($+10$).

\begin{figure}[h]
	\centering
	\begin{minipage}{.35\linewidth}
		\centering
		\begin{subfigure}{\textwidth}
			\includegraphics[width=\textwidth]{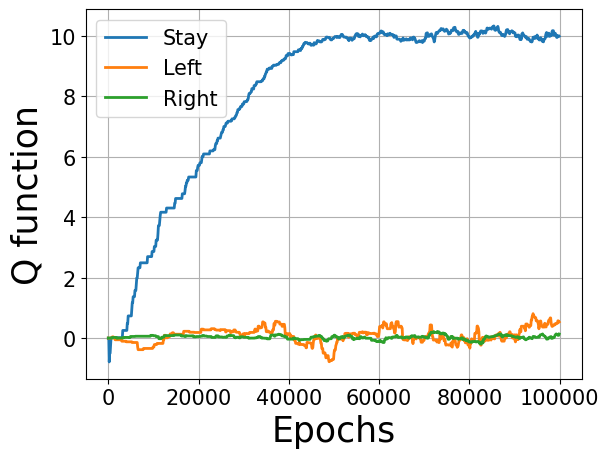}
			\subcaption{Leftmost position and $t=3$}
			\label{fig:logic_0}
		\end{subfigure}
	\end{minipage}
	\hspace{20mm}
	\begin{minipage}{.35\linewidth}
		\centering
		\begin{subfigure}{\textwidth}
			\includegraphics[width=\textwidth]{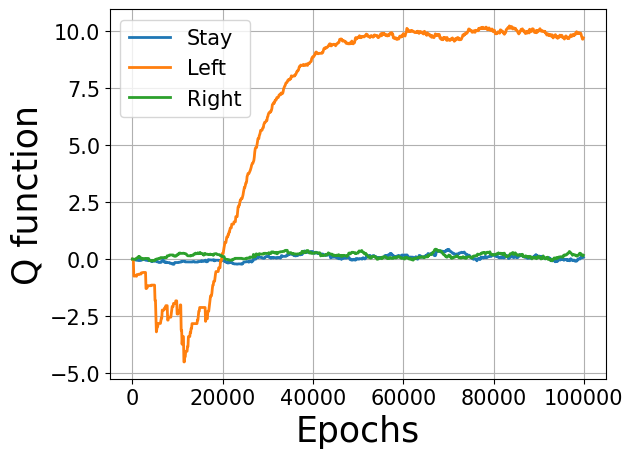}
			\subcaption{Slot adjacent to leftmost and $t=2$}
			\label{fig:logic_1}
		\end{subfigure}
	\end{minipage}
	\begin{minipage}{.35\linewidth}
		\centering
		\begin{subfigure}{\textwidth}
			\includegraphics[width=\textwidth]{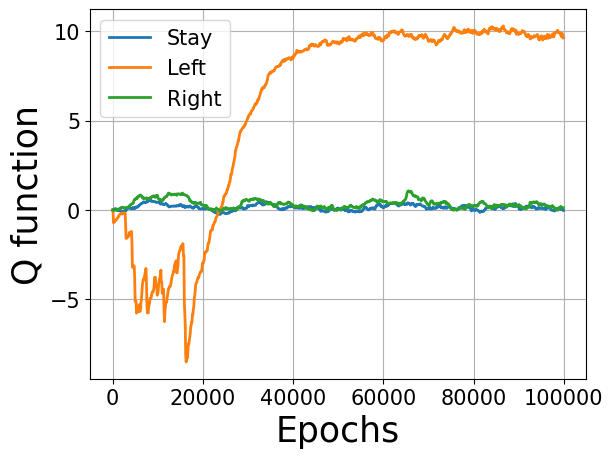}
			\subcaption{Slot adjacent to rightmost and $t=1$}
			\label{fig:logic_2}
		\end{subfigure}
	\end{minipage}
	\hspace{20mm}
	\begin{minipage}{.35\linewidth}
		\centering
		\begin{subfigure}{\textwidth}
			\includegraphics[width=\textwidth]{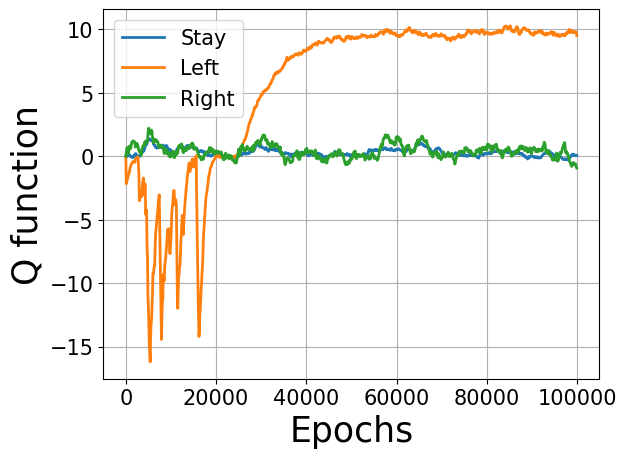}
			\subcaption{Rightmost position and $t=0$}
			\label{fig:logic_3}
		\end{subfigure}
	\end{minipage}
	\caption{Learning curves for agent 2 of \textit{Logical Team Q-learning} for a random seed.}
	\label{fig:logic_q_values}
\end{figure}

Figures \ref{fig:iql_values} show the learning curves for IQL. Note that IQL fails at this environment because it has no mechanism to discard the $-30$ penalty incurred due to moving to the left when agent 1 presses the button due to exploration.

\begin{figure}[h]
	\centering
	\begin{minipage}{.35\linewidth}
		\centering
		\begin{subfigure}{\textwidth}
			\includegraphics[width=\textwidth]{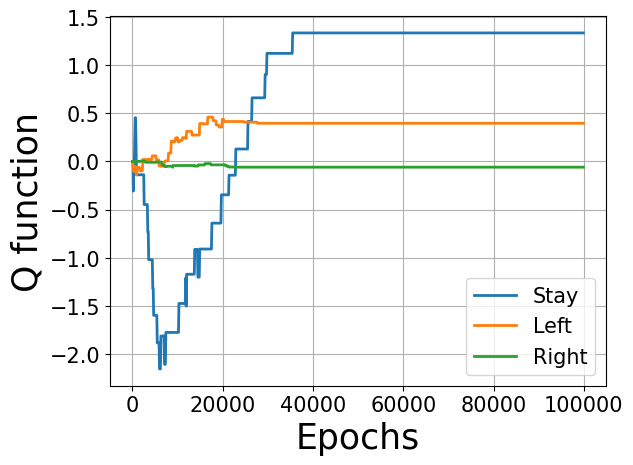}
			\subcaption{Leftmost position and $t=3$}
			\label{fig:iql_0}
		\end{subfigure}
	\end{minipage}
	\hspace{20mm}
	\begin{minipage}{.35\linewidth}
		\centering
		\begin{subfigure}{\textwidth}
			\includegraphics[width=\textwidth]{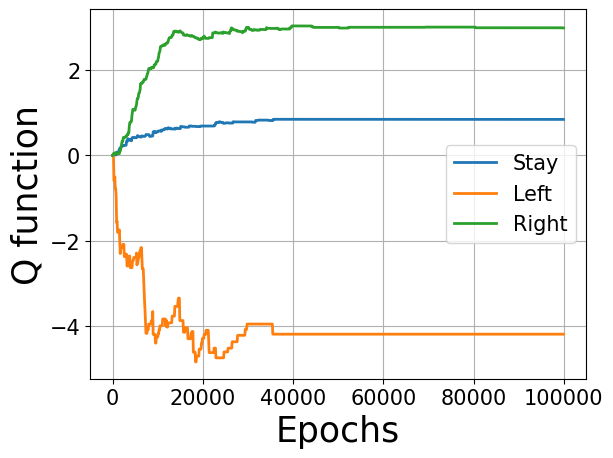}
			\subcaption{Slot adjacent to leftmost and $t=2$}
			\label{fig:iql_1}
		\end{subfigure}
	\end{minipage}
	\begin{minipage}{.35\linewidth}
		\centering
		\begin{subfigure}{\textwidth}
			\includegraphics[width=\textwidth]{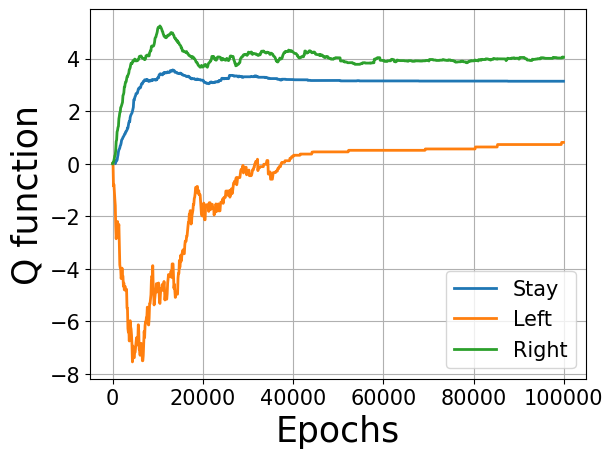}
			\subcaption{Slot adjacent to rightmost and $t=1$}
			\label{fig:iql_2}
		\end{subfigure}
	\end{minipage}
	\hspace{20mm}
	\begin{minipage}{.35\linewidth}
		\centering
		\begin{subfigure}{\textwidth}
			\includegraphics[width=\textwidth]{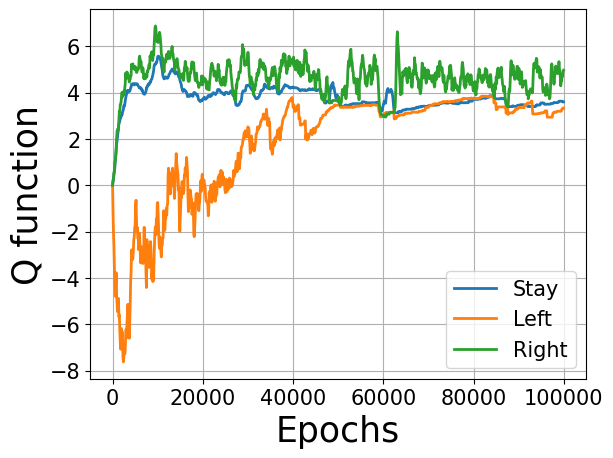}
			\subcaption{Rightmost position and $t=0$}
			\label{fig:iql_3}
		\end{subfigure}
	\end{minipage}
	\caption{Learning curves for agent 2 of IQL  for a random seed.}
	\label{fig:iql_values}
\end{figure}

Figures \ref{fig:dist_q_values} show the learning curves for DistQ. The reason that this algorithm cannot solve this environment is that it severely overestimates the value of choosing to move to the right whilst on the rightmost position. It is well known that this is a consequence of the fact that DistQ only performs updates that increase the estimates of the $Q$-values combined with the stochastic reward received when agent 2 ``stumbles" against the right edge.

\begin{figure}[h]
	\centering
	\begin{minipage}{.35\linewidth}
		\centering
		\begin{subfigure}{\textwidth}
			\includegraphics[width=\textwidth]{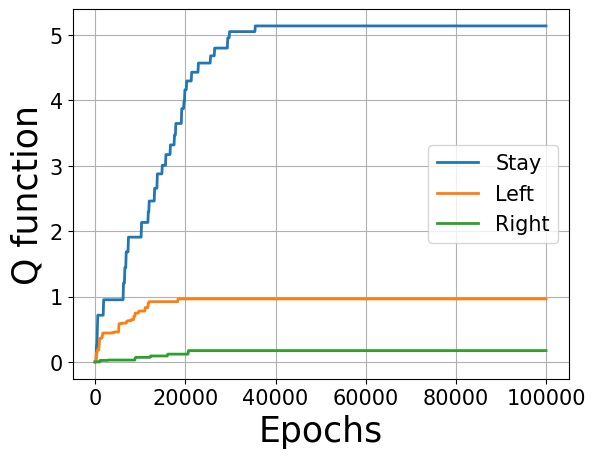}
			\subcaption{Leftmost position and $t=3$}
			\label{fig:dist_0}
		\end{subfigure}
	\end{minipage}
	\hspace{20mm}
	\begin{minipage}{.35\linewidth}
		\centering
		\begin{subfigure}{\textwidth}
			\includegraphics[width=\textwidth]{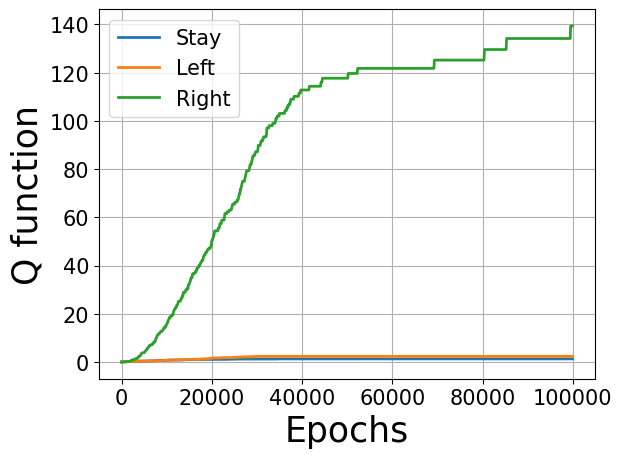}
			\subcaption{Slot adjacent to leftmost and $t=2$}
			\label{fig:dist_1}
		\end{subfigure}
	\end{minipage}
	\begin{minipage}{.35\linewidth}
		\centering
		\begin{subfigure}{\textwidth}
			\includegraphics[width=\textwidth]{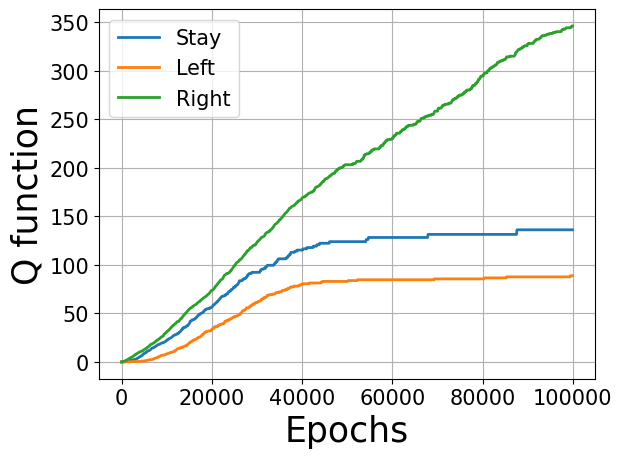}
			\subcaption{Slot adjacent to rightmost and $t=1$}
			\label{fig:dist_2}
		\end{subfigure}
	\end{minipage}
	\hspace{20mm}
	\begin{minipage}{.35\linewidth}
		\centering
		\begin{subfigure}{\textwidth}
			\includegraphics[width=\textwidth]{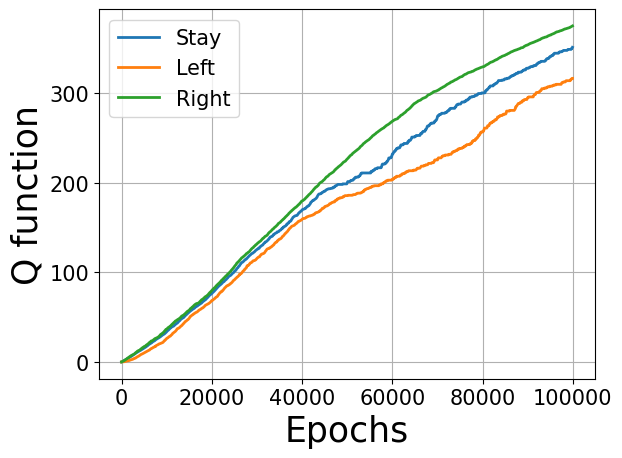}
			\subcaption{Rightmost position and $t=0$}
			\label{fig:dist_3}
		\end{subfigure}
	\end{minipage}
	\caption{Learning curves for agent 2 of DistQ  for a random seed.}
	\label{fig:dist_q_values}
\end{figure}

Figures \ref{fig:hyst_q_values} show the learning curves for HystQ. This algorithm cannot solve this environment because it has two issues and the way to solve one makes the other worse. More specifically, one can be solved by increasing the smaller step-size, while the other needs to decrease it. The first issue is the same one that affects DistQ, i.e., the overestimation of the \textit{move right} action in the rightmost position. Note that this can be ameliorated by increasing the small step-size. The second issue is the penalty incurred due to moving to the left when agent 1 presses the button. This can be ameliorated by decreasing the small step-size. The fact that there is no intermediate value for the small step-size to solve both issues is the reason that this algorithm cannot solve this environment.

\begin{figure}[h]
	\centering
	\begin{minipage}{.35\linewidth}
		\centering
		\begin{subfigure}{\textwidth}
			\includegraphics[width=\textwidth]{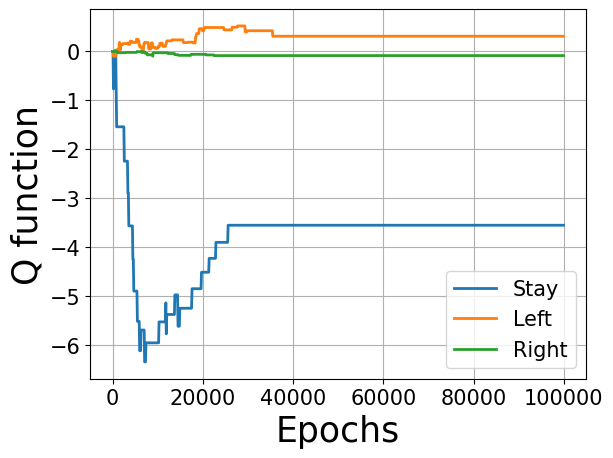}
			\subcaption{Leftmost position and $t=3$}
			\label{fig:hyst_0}
		\end{subfigure}
	\end{minipage}
	\hspace{20mm}
	\begin{minipage}{.35\linewidth}
		\centering
		\begin{subfigure}{\textwidth}
			\includegraphics[width=\textwidth]{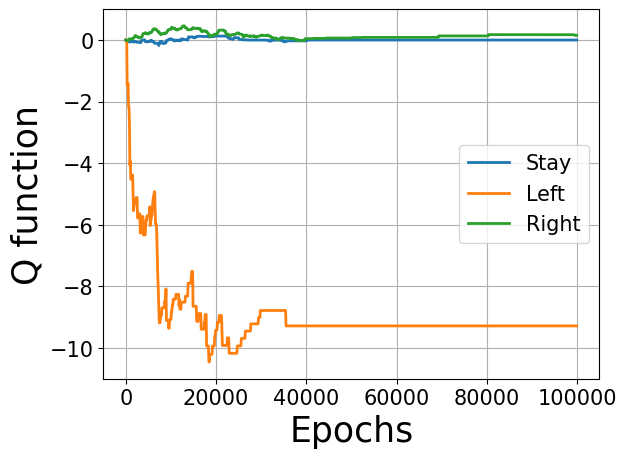}
			\subcaption{Slot adjacent to leftmost and $t=2$}
			\label{fig:hyst_1}
		\end{subfigure}
	\end{minipage}
	\begin{minipage}{.35\linewidth}
		\centering
		\begin{subfigure}{\textwidth}
			\includegraphics[width=\textwidth]{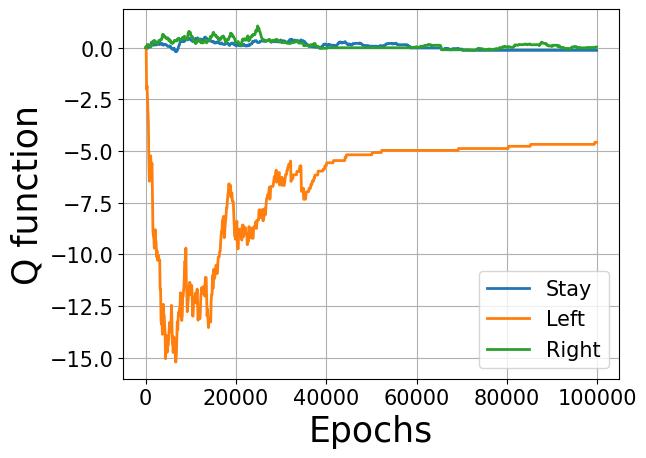}
			\subcaption{Slot adjacent to rightmost and $t=1$}
			\label{fig:hyst_2}
		\end{subfigure}
	\end{minipage}
	\hspace{20mm}
	\begin{minipage}{.35\linewidth}
		\centering
		\begin{subfigure}{\textwidth}
			\includegraphics[width=\textwidth]{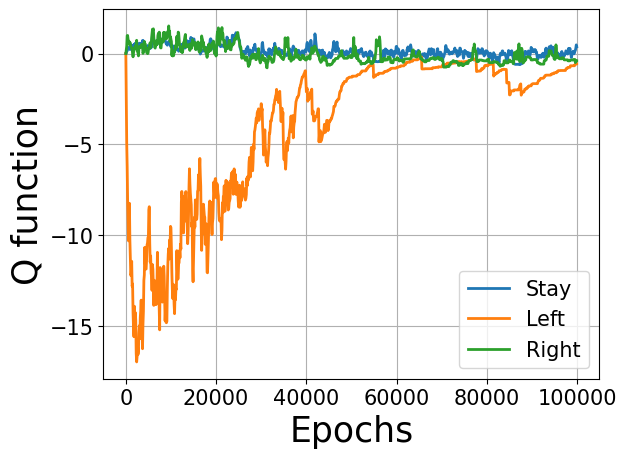}
			\subcaption{Rightmost position and $t=0$}
			\label{fig:hyst_3}
		\end{subfigure}
	\end{minipage}
	\caption{Learning curves for agent 2 of HystQ  for a random seed.}
	\label{fig:hyst_q_values}
\end{figure}

Figures \ref{fig:qmix_q_values} show the learning curves for Qmix. Qmix fails at this task due to the fact that its monotonic factoring assumption is not satisfied at this task. The architecture used is as follows: we used tabular representation for the individual $q$ functions, and for the mixing and hypernetworks we used the architecture specified in \cite{rashid2020monotonic}. More specifically, the mixing network is composed of two hidden layers (with $10$ units each) with \textit{ELu} nonlinearities in the first layer while the second layer is linear. The hypernetworks that output the weights of the mixing network consist of two layers with $ReLU$ nonlinearities followed by an activation function that takes the absolute value to ensure that the mixing network weights are non-negative. The bias of the first mixing layer is produced by a network with a unique linear layer and the other bias is produced by a two layer hypernetwork with a $ReLU$ nonlinearity. All hypernetwork layers are fully connected and have $5$ units.
\begin{figure}[h]
	\centering
	\begin{minipage}{.35\linewidth}
		\centering
		\begin{subfigure}{\textwidth}
			\includegraphics[width=\textwidth]{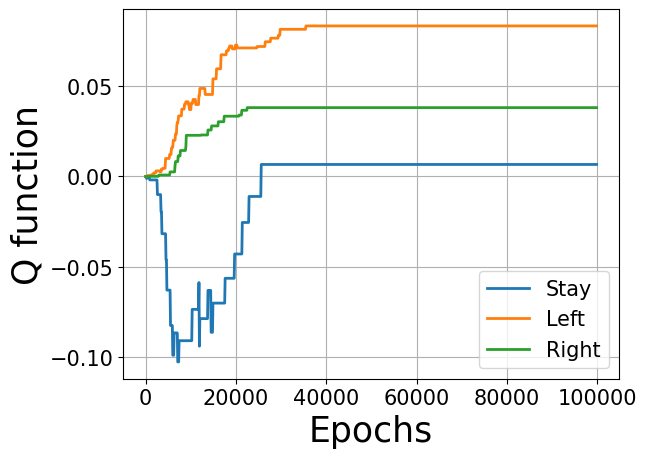}
			\subcaption{Leftmost position and $t=3$}
			\label{fig:qmix_0}
		\end{subfigure}
	\end{minipage}
	\hspace{20mm}
	\begin{minipage}{.35\linewidth}
		\centering
		\begin{subfigure}{\textwidth}
			\includegraphics[width=\textwidth]{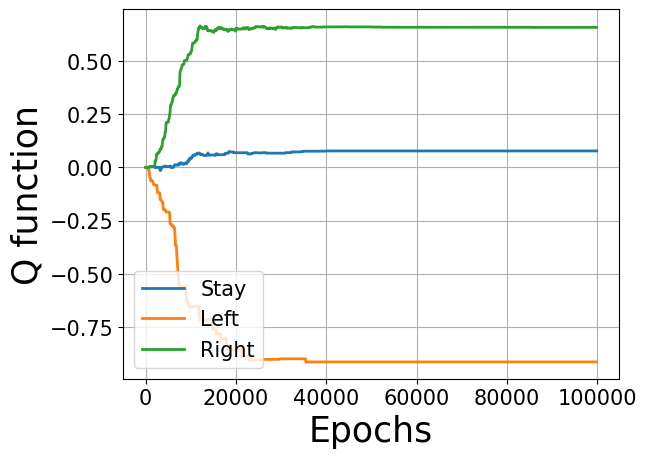}
			\subcaption{Slot adjacent to leftmost and $t=2$}
			\label{fig:qmix_1}
		\end{subfigure}
	\end{minipage}
	\begin{minipage}{.35\linewidth}
		\centering
		\begin{subfigure}{\textwidth}
			\includegraphics[width=\textwidth]{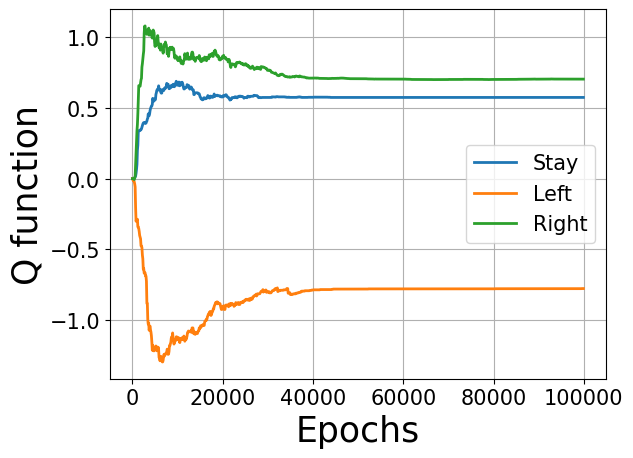}
			\subcaption{Slot adjacent to rightmost and $t=1$}
			\label{fig:qmix_2}
		\end{subfigure}
	\end{minipage}
	\hspace{20mm}
	\begin{minipage}{.35\linewidth}
		\centering
		\begin{subfigure}{\textwidth}
			\includegraphics[width=\textwidth]{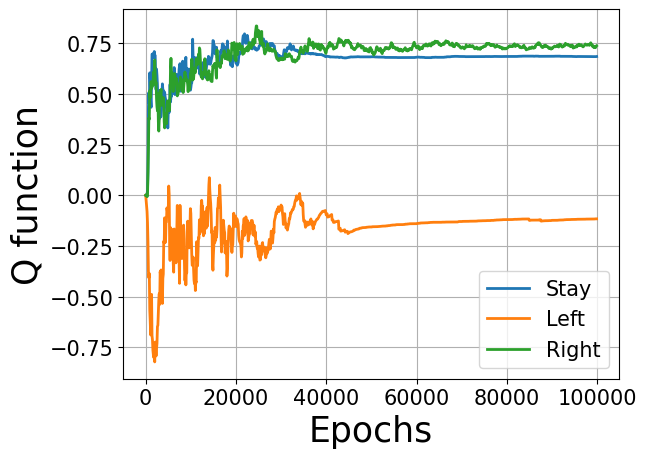}
			\subcaption{Rightmost position and $t=0$}
			\label{fig:qmix_3}
		\end{subfigure}
	\end{minipage}
	\caption{Learning curves for agent 2 of Qmix  for a random seed.}
	\label{fig:qmix_q_values}
\end{figure}

Figures \ref{fig:qtran_q_values} show the learning curves for Qtran. Qtran succeeds at this task. However, it is important to remark that this is a tabular implementation of Qtran (an algorithm designed to be used in conjuntion with NNs in complex environment), where the algorithm estimates the full joint q-function in tabular form, which is not scalable and defeats the purpose of learning factored q-functions.

\begin{figure}[H]
	\centering
	\begin{minipage}{.35\linewidth}
		\centering
		\begin{subfigure}{\textwidth}
			\includegraphics[width=\textwidth]{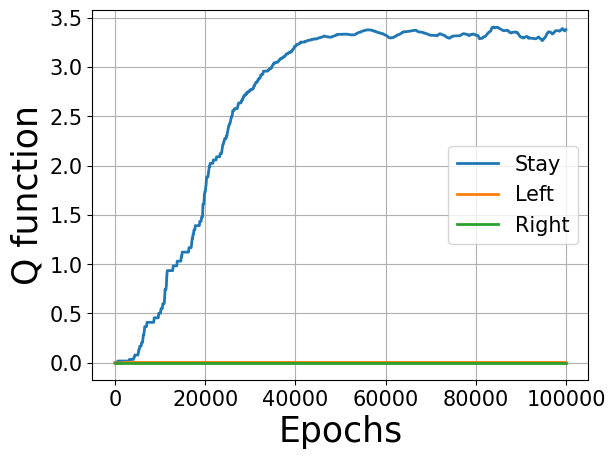}
			\subcaption{Leftmost position and $t=3$}
			\label{fig:qtran_0}
		\end{subfigure}
	\end{minipage}
	\hspace{20mm}
	\begin{minipage}{.35\linewidth}
		\centering
		\begin{subfigure}{\textwidth}
			\includegraphics[width=\textwidth]{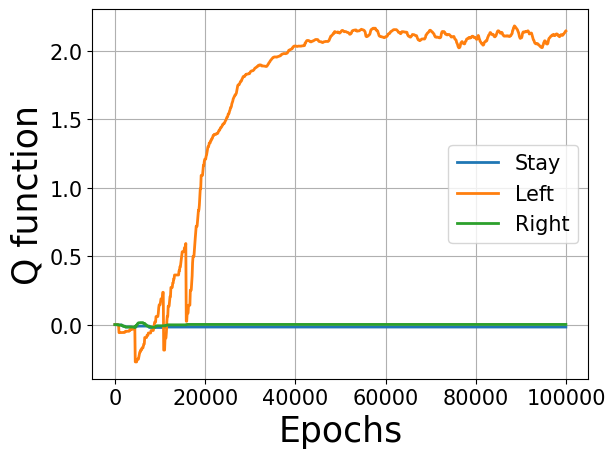}
			\subcaption{Slot adjacent to leftmost and $t=2$}
			\label{fig:qtran_1}
		\end{subfigure}
	\end{minipage}
	\begin{minipage}{.35\linewidth}
		\centering
		\begin{subfigure}{\textwidth}
			\includegraphics[width=\textwidth]{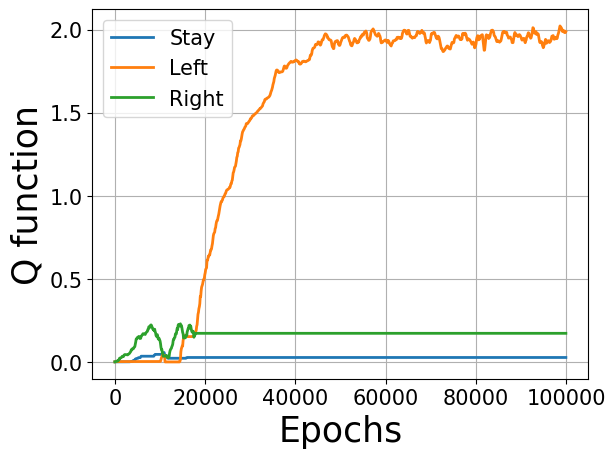}
			\subcaption{Slot adjacent to rightmost and $t=1$}
			\label{fig:qtran_2}
		\end{subfigure}
	\end{minipage}
	\hspace{20mm}
	\begin{minipage}{.35\linewidth}
		\centering
		\begin{subfigure}{\textwidth}
			\includegraphics[width=\textwidth]{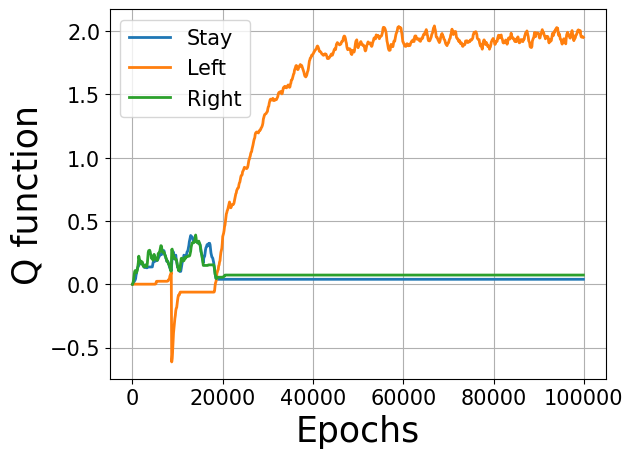}
			\subcaption{Rightmost position and $t=0$}
			\label{fig:qtran_3}
		\end{subfigure}
	\end{minipage}
	\caption{Learning curves for agent 2 of Qtran  for a random seed.}
	\label{fig:qtran_q_values}
\end{figure}

All code is available at https://github.com/lcassano/Logical-Team-Q-Learning-paper.

\subsection{Cowboy bull game additional results}\label{app:experiment3}
The bull's policy is given by the pseudocode shown in algorithm \ref{algorithm:policy_bull}.
\begin{algorithm}[H]
	\caption{Bull's policy.}
	\label{algorithm:policy_bull}
	\begin{algorithmic}
		\IF{distance to all predators $>$ 10 (this circumference is depicted by the blue line in figure \ref{fig:cowboy_game})}
		\STATE Natural foraging behavior: Stay still with 90\% probability, otherwise make a small move in a random direction.
		\ELSE
		\IF{the maximum angle formed by two predators is $>$ $108^o$}
		\STATE There's a hole to escape: Escape through the direction in between these two predators.
		\ELSIF{distance to farthest predator - distance to closest predator $>$ 5}
		\STATE There's no hole, but one predator is much closer than the others so run in the direction opposite to this predator.
		\ELSE
		\STATE No way out (scared): Stay still with 70\% probability, otherwise make a fast move in a random direction.
		\ENDIF
		\ENDIF
	\end{algorithmic}
\end{algorithm}

We now specify the hyperparameters for \textit{Logical Team Q-learning}. All NN's have two hidden layers with $50$ units and ReLu nonlinearities. However, for each $Q$-network, instead of having one network with $5$ outputs, we have $5$ networks each with $1$ output (one for each action). At every epoch the agent collects data by playing $32$ full games and then performs $50$ gradient backpropagation steps. Half of the $32$ games are played greedily and the other half use a Boltzmann policy with temperature $b_T$ that decays according to the following schedule $b_T=\max[0.05, 0.5\times(1-\textit{epoch}/15\times10^3)]$. We use this behavior policy to ensure that there are sufficient transitions that satisfy condition $c_1$ and that also there are transitions that satisfy $c_2$. The target networks are updated every $50$ backprop steps. The capacity of the replay buffer is $2.10^5$ transitions, the mini-batch size is 1024, $\alpha=1$, we use a discount factor equal to $0.99$ and optimize the networks using the Adam optimizer with initial step-size $10^{-5}$.

The hyperparameters of the HystQ implementation are the same as those of LTQL, the ratio of the two step-sizes used by HystQ is $0.1$. To run IQL we used the implementation of HystQ with the ratio of the two step-sizes set to $0$.

The architecture used by Qmix is the one suggested in \cite{rashid2020monotonic} with the exception that, for fairness, the individual $Q$-networks used the same architecture as the ones used by the other algorithms (i.e., $5$ networks with a unique output as opposed to $1$ network with $5$ outputs). All hidden layers of the hypernetworks as well the mixing network have $10$ units. In this case we did $5$ backprop iterations per epoch and the target network update period is $15$. We use a batch size of $256$, a discount factor equal to $0.98$ and optimize the networks with the Adam optimizer with initial step-size $10^{-6}$. In this case the behavior policy is always Boltzmann with the following annealing schedule for the temperature parameter $b_T=\max[0.005, 0.05\times(1-\textit{epoch}/25\times10^3)]$.

The Qtran-base variant was implemented. The individual $Q$-networks have the same architecture used by the other algorithms. The joint $Q$-network has two hidden layers with $60$ units each, the input of this network is the global state concatenated with the agents' actions with one hot encoding. The value function network has only one hidden layer with $25$ units and its input is the global state. The target networks are updated every $50$ backprop steps. The capacity of the replay buffer is $2.10^5$ transitions, the mini-batch size is 1024, we use a discount factor equal to $0.99$ and optimize the networks using the Adam optimizer with initial step-size $10^{-5}$.

The batch size, Boltzmann temperature value, learning step-size and target update period were chosen by grid search.

All implementations use TensorFlow 2. The code is available at https://github.com/lcassano/Logical-Team-Q-Learning-paper. Running one seed for one agent takes approximately 12 hours in our hardware (2017 iMac with 3.8 GHz Intel Core i5 and 16GB of RAM).
\end{document}